\documentclass{article}

\usepackage{hyperref}

\usepackage{amsmath}
\usepackage{amssymb}
\usepackage{mathtools}
\usepackage{amsthm}

\usepackage[accepted]{icml_template/icml2024}

\usepackage[capitalize,noabbrev]{cleveref}

\theoremstyle{plain}
\newtheorem{theorem}{Theorem}[section]

\theoremstyle{definition}
\newtheorem{definition}[theorem]{Definition}

\usepackage{color, colortbl}
\definecolor{LightCyan}{rgb}{0.94,1,1}

\usepackage[textsize=tiny]{todonotes}

\icmltitlerunning{\lowrank{}: Memory-Efficient LLM Training by Gradient Low-Rank Projection}

\usepackage{multirow}
\usepackage{xcolor}
\usepackage{fancyhdr,graphicx}
\usepackage{amssymb,amsthm,amsmath}
\usepackage{booktabs,array} %
\usepackage{thmtools}
\usepackage{thm-restate}
\usepackage[title]{appendix}
\usepackage{tikz}
\usepackage{resizegather}
\usepackage{wrapfig}
\usepackage{enumitem,kantlipsum}
\usepackage{mathtools}
\usepackage{nccmath}

 \usepackage{xspace}

\usepackage{dsfont}
\usepackage{stmaryrd}

\renewcommand{\phi}{\varphi}

\newcommand{\rank}{\operatorname{rank}}

\newcommand{\diag}{\operatorname{diag}}

\usepackage{amsmath,amsfonts,bm}

\def\eqref#1{equation~\ref{#1}}

\def\1{\bm{1}}

\def\rr{{\textnormal{r}}}

\def\vone{{\bm{1}}}
\def\vmu{{\bm{\mu}}}

\def\va{{\bm{a}}}

\def\ve{{\bm{e}}}
\def\vf{{\bm{f}}}
\def\vg{{\bm{g}}}
\def\vh{{\bm{h}}}

\def\vu{{\bm{u}}}
\def\vv{{\bm{v}}}
\def\vw{{\bm{w}}}
\def\vx{{\bm{x}}}
\def\vy{{\bm{y}}}
\def\vz{{\bm{z}}}

\DeclareMathAlphabet{\mathsfit}{\encodingdefault}{\sfdefault}{m}{sl}
\SetMathAlphabet{\mathsfit}{bold}{\encodingdefault}{\sfdefault}{bx}{n}

\newtheorem{property}{Property}
\usepackage{graphicx}
\usepackage{pifont}
\usepackage{xcolor,colortbl}
\usepackage{multicol}
\usepackage{tablefootnote}
\usepackage[para,online,flushleft]{threeparttable}
\newcommand{\cmark}{\ding{51}}%
\newcommand{\xmark}{\ding{55}}%
\usepackage{xspace}

\usepackage{setspace}

\usepackage[many]{tcolorbox}

\usepackage{subcaption}

\usepackage{letltxmacro}

\usepackage{siunitx}
\usepackage{xcolor}
\usepackage{listings}

\usepackage[ruled,vlined]{algorithm2e}
\definecolor{commentcolor}{RGB}{110,154,155}   %
\newcommand{\PyComment}[1]{\ttfamily\textcolor{commentcolor}{\# #1}}  %
\newcommand{\PyCode}[1]{\ttfamily\textcolor{black}{#1}} %

\newcommand{\lowrank}{GaLore\xspace}
\newcommand{\update}{\textsc{update}}

\begin{document}

\twocolumn[
\icmltitle{\lowrank{}: Memory-Efficient LLM Training by Gradient Low-Rank Projection}

\icmlsetsymbol{equal}{*}

\icmlsetaffilorder{caltech,meta,utaustin,cmu}

\begin{icmlauthorlist}
\icmlauthor{Jiawei Zhao}{caltech}
\icmlauthor{Zhenyu Zhang}{utaustin}
\icmlauthor{Beidi Chen}{meta,cmu}
\icmlauthor{Zhangyang Wang}{utaustin}
\icmlauthor{Anima Anandkumar}{equal,caltech}
\icmlauthor{Yuandong Tian}{equal,meta}
\end{icmlauthorlist}

\icmlaffiliation{caltech}{California Institute of Technology}
\icmlaffiliation{meta}{Meta AI}
\icmlaffiliation{utaustin}{University of Texas at Austin}
\icmlaffiliation{cmu}{Carnegie Mellon University}

\icmlcorrespondingauthor{Jiawei Zhao}{jiawei@caltech.edu}
\icmlcorrespondingauthor{Yuandong Tian}{yuandong@meta.com}

\vskip 0.3in
]

\printAffiliationsAndNotice{\icmlEqualContribution} %

\begin{abstract}
    Training Large Language Models (LLMs) presents significant memory challenges, predominantly due to the growing size of weights and optimizer states.
    Common memory-reduction approaches, such as low-rank adaptation (LoRA), add a trainable low-rank matrix to the frozen pre-trained weight in each layer. However, such approaches typically underperform training with full-rank weights in both pre-training and fine-tuning stages since they limit the parameter search to a low-rank subspace and alter the training dynamics, and further, may require full-rank warm start. 
    In this work, we propose Gradient Low-Rank Projection (\textbf{\lowrank}), a training strategy that allows \emph{full-parameter} learning but is more \emph{memory-efficient} than common low-rank adaptation methods such as LoRA.
    Our approach reduces memory usage by up to 65.5\% in optimizer states while maintaining both efficiency and performance for pre-training on LLaMA 1B and 7B architectures with C4 dataset with up to 19.7B tokens, and on fine-tuning RoBERTa on GLUE tasks.
    Our 8-bit \lowrank{} further reduces optimizer memory by up to 82.5\% and total training memory by 63.3\%, compared to a BF16 baseline.
    Notably, we demonstrate, for the first time, the feasibility of pre-training a 7B model on consumer GPUs with 24GB memory (e.g., NVIDIA RTX 4090) without model parallel, checkpointing, or offloading strategies. Code is provided in the \href{https://github.com/jiaweizzhao/GaLore}{link}.
\end{abstract}

\vspace{-8mm}
\section{Introduction}
Large Language Models (LLMs) have shown impressive performance across multiple disciplines, including conversational AI and language translation.
However, pre-training and fine-tuning LLMs require not only a huge amount of computation but is also memory intensive. The memory requirements include not only billions of trainable parameters,  but also their gradients and optimizer states (e.g., gradient momentum and variance in Adam) that can be larger than parameter storage themselves  ~\citep{raffelExploringLimitsTransfer2023,touvronLlamaOpenFoundation2023,chowdheryPaLMScalingLanguage2022}. For example, pre-training a LLaMA 7B model from scratch with a single batch size requires at least 58 GB memory (14GB for trainable parameters, 42GB for Adam optimizer states and weight gradients, and 2GB for activations\protect\footnotemark[1]).
This makes the training not feasible on consumer-level GPUs such as NVIDIA RTX 4090 with 24GB memory.

\definecolor{cus_color}{rgb}{22, 96, 55}

\makeatletter
\newcommand{\removelatexerror}{\let\@latex@error\@gobble}
\makeatother

\begin{figure}[!t]
\centering
\includegraphics[width=1\columnwidth]{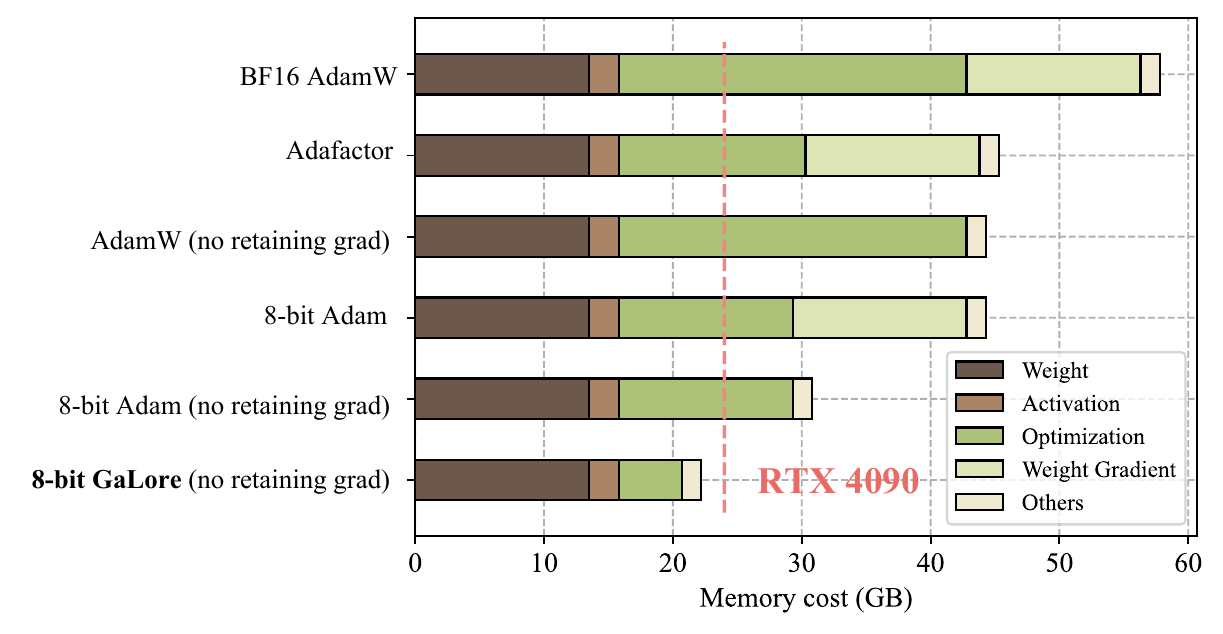}
\vskip -0.11in
\caption{\small{Estimated memory consumption of pre-training a LLaMA 7B model with a token batch size of 256 on a single device, without activation checkpointing and memory offloading\protect\footnotemark[2]. Details refer to Section~\ref{sec:memory_measure}.}}
\label{fig:memory_breakdown} 
\vskip -0.15in
\end{figure}
\footnotetext[1]{The calculation is based on LLaMA architecture, BF16 numerical format, and maximum sequence length of 2048.}
\footnotetext[2]{In the figure, ``no retaining grad'' denotes the application of per-layer weight update to reduce memory consumption of storing weight gradient \citep{lvFullParameterFinetuning2023}.}
\SetAlFnt{\fontsize{8pt}{9pt}\selectfont}
\SetAlCapFnt{\fontsize{8pt}{9pt}\selectfont}
\begin{algorithm}[t]
    \SetAlgoLined
        \PyCode{for weight in model.parameters():} \\
        \Indp   %
            \PyCode{grad = weight.grad} \\ 
            \PyComment{original space -> compact space} \\
            \PyCode{lor\_grad = \textbf{project}(grad)} \\
            \PyComment{update by Adam, Adafactor, etc.} \\
            \PyCode{lor\_update = \textbf{update}(lor\_grad)} \\
            \PyComment{compact space -> original space} \\
            \PyCode{update = \textbf{project\_back}(lor\_update)} \\
            \PyCode{weight.data += update} \\
        \Indm %
    \caption{\fontsize{8pt}{9pt}\selectfont{\lowrank{}, PyTorch-like}}
    \label{alg:code_box}
\end{algorithm}

In addition to engineering and system efforts, such as gradient checkpointing~\cite{chenTrainingDeepNets2016}, memory offloading~\cite{rajbhandariZeROMemoryOptimizations2020}, etc., to achieve faster and more efficient distributed training, researchers also seek to develop various optimization techniques to reduce the memory usage during pre-training and fine-tuning. 

\def\rr{\mathbb{R}}

Parameter-efficient fine-tuning (PEFT) techniques allow for the efficient adaptation of pre-trained language models (PLMs) to different downstream applications without the need to fine-tune all of the model's parameters \citep{dingDeltaTuningComprehensive2022}.
Among them, the popular Low-Rank Adaptation (LoRA~\citet{huLoRALowRankAdaptation2021}) \emph{reparameterizes} weight matrix $W\in \rr^{m\times n}$ into $W = W_0 + BA$, where $W_0$ is a frozen full-rank matrix and $B\in\rr^{m\times r}$, $A\in\rr^{r\times n}$ are additive low-rank adaptors to be learned. Since the rank $r \ll \min(m,n)$, $A$ and $B$ contain fewer number of trainable parameters and thus smaller optimizer states. LoRA has been used extensively to reduce memory usage for fine-tuning in which $W_0$ is the frozen pre-trained weight. Its variant ReLoRA is also used in pre-training, by periodically updating $W_0$ using previously learned low-rank adaptors \citep{lialinReLoRAHighRankTraining2023}.

However, many recent works demonstrate the limitation of such a low-rank reparameterization. For fine-tuning, LoRA is not shown to reach a comparable performance as full-rank fine-tuning~\cite{xiaChainLoRAEfficient2024}.
For pre-training from scratch, it is shown to require a full-rank model training as a warmup \citep{lialinReLoRAHighRankTraining2023}, before optimizing in the low-rank subspace. There are two possible reasons: (1) the optimal weight matrices may not be low-rank, and (2) the reparameterization changes the gradient training dynamics.

{\bf Our approach: }To address the above challenge, we propose Gradient Low-Rank Projection (\textbf{\lowrank}), a training strategy that allows \emph{full-parameter} learning but is more \emph{memory-efficient} than common low-rank adaptation methods, such as LoRA. Our key idea is to leverage the slow-changing low-rank structure of the \emph{gradient} $G\in\rr^{m\times n}$ of the weight matrix $W$, rather than trying to approximate the weight matrix itself as low rank. 

We first show theoretically that the gradient matrix $G$ becomes low-rank during training. Then, we propose \lowrank{} that computes two projection matrices $P\in \rr^{m\times r}$ and $Q\in \rr^{n\times r}$ to project the gradient matrix $G$ into a low-rank form $P^\top G Q$.
In this case, the memory cost of optimizer states, which rely on component-wise gradient statistics, can be substantially reduced. Occasional updates of $P$ and $Q$ (e.g., every 200 iterations) incur minimal amortized additional computational cost.
\lowrank is more memory-efficient than LoRA as shown in Table~\ref{tab:lora_compare}.
In practice, this yields up to 30\% memory reduction compared to LoRA during pre-training.

We demonstrate that \lowrank{} works well in both LLM pre-training and fine-tuning. When pre-training LLaMA 7B on C4 dataset, 8-bit \lowrank, combined with 8-bit optimizers and layer-wise weight updates techniques, achieves comparable performance to its full-rank counterpart, with less than 10\% memory cost of optimizer states. 

Notably, for pre-training, \lowrank{} keeps low memory throughout the entire training, without requiring full-rank training warmup like ReLoRA.  
Thanks to \lowrank's memory efficiency, it is possible to train LLaMA 7B from scratch on a single GPU with 24GB memory (e.g., on NVIDIA RTX 4090), without any costly memory offloading techniques (Fig.~\ref{fig:memory_breakdown}).

\lowrank{} is also used to fine-tune pre-trained LLMs on GLUE benchmarks with comparable or better results than existing low-rank methods. 
When fine-tuning RoBERTa-Base on GLUE tasks with a rank of 4, \lowrank{} achieves an average score of 85.89, outperforming LoRA, which achieves a score of 85.61.

As a gradient projection method, \lowrank{} is independent of the choice of optimizers and can be easily plugged into existing ones with only two lines of code, as shown in Algorithm~\ref{alg:code_box}. Our experiment (Fig.~\ref{fig:compare_optimizer}) shows that it works for popular optimizers such as AdamW, 8-bit Adam, and Adafactor. In addition, its performance is insensitive to very few hyper-parameters it introduces. We also provide theoretical justification on the low-rankness of gradient update, as well as the convergence analysis of \lowrank.

\section{Related Works}
\paragraph{Low-rank adaptation.}
\citet{huLoRALowRankAdaptation2021} proposed Low-Rank Adaptation (LoRA) to fine-tune pre-trained models with low-rank adaptors. 
This method reduces the memory footprint by maintaining a low-rank weight adaptor for each layer.
There are a few variants of LoRA proposed to enhance its performance \citep{renduchintalaTiedLoraEnhacingParameter2023, shengSLoRAServingThousands2023, zhangLORAFAMEMORYEFFICIENTLOWRANK, xiaChainLoRAEfficient2024}, supporting multi-task learning \citep{wangMultiLoRADemocratizingLoRA2023}, and further reducing the memory footprint \citep{dettmersQLoRAEfficientFinetuning2023}.
\citet{lialinReLoRAHighRankTraining2023} proposed ReLoRA, a variant of LoRA designed for pre-training, but requires a full-rank training warmup to achieve comparable performance as the standard baseline.
Inspired by LoRA, \citet{haoFloraLowRankAdapters2024} also suggested that gradients can be compressed in a low-rank subspace, and they proposed to use random projections to compress the gradients.
There have also been approaches that propose training networks with low-rank factorized weights from scratch \citep{kamalakaraExploringLowRank2022,wangCuttlefishLowrankModel2023,zhaoInRankIncrementalLowRank2023}.

\paragraph{Subspace learning.}
Recent studies have demonstrated that the learning primarily occurs within a significantly low-dimensional parameter subspace \citep{gur-ariGradientDescentHappens2018,larsenHowManyDegrees2022}.
These findings promote a special type of learning called \textit{subspace learning}, where the model weights are optimized within a low-rank subspace. 
This notion has been widely used in different domains of machine learning, including meta-learning and continual learning \citep{leeGradientBasedMetaLearningLearned2018,chaudhryContinualLearningLowrank2020}.

\paragraph{Projected gradient descent.}
\lowrank{} is closely related to the traditional topic of projected gradient descent (PGD) \citep{chenFastLowrankEstimation2015, chenNonConvexProjectedGradient2019}. 
A key difference is that, \lowrank{} considers the specific gradient form that naturally appears in training multi-layer neural networks (e.g., it is a matrix with specific structures), proving many of its properties (e.g., Lemma~\ref{lemma:gradientlowrank}, Theorem~\ref{thm:gradientreversible}, and Theorem~\ref{thm:convgpg}). In contrast, traditional PGD mostly treats the objective as a general blackbox nonlinear function, and study the gradients in the vector space only. 

\paragraph{Low-rank gradient.}
Gradient is naturally low-rank during training of neural networks, and this property have been studied in both theory and practice \citep{zhaoZerOInitializationInitializing2022,cossonLowRankGradientDescent2023,yang2023spectral}.
It has been applied to reduce communication cost \citep{wangATOMOCommunicationefficientLearning,vogelsPowerGossipPracticalLowRank2020}, and memory footprint during training \citep{gooneratneLowrankGradientApproximation2020,huangLowRankGradientDescent2023,modoranuErrorFeedbackCan2024}.

\paragraph{Memory-efficient optimization.}
There have been some works trying to reduce the memory cost of gradient statistics for adaptive optimization algorithms \citep{shazeerAdafactorAdaptiveLearning,anilMemoryEfficientAdaptive,dettmers8bitOptimizersBlockwise2021}. 
Quantization is widely used to reduce the memory cost of optimizer states \citep{dettmers8bitOptimizersBlockwise2021,liMemoryEfficientOptimizers2023}.
Recent works have also proposed to reduce weight gradient memory by fusing the backward operation with the optimizer update \citep{lvAdaLomoLowmemoryOptimization2023,lvFullParameterFinetuning2023}.

\def\cG{\mathcal{G}}
\def\rr{\mathbb{R}}

\section{\lowrank: Gradient Low-Rank Projection}
\subsection{Background}

\textbf{Regular full-rank training.} At time step $t$, $G_t = -\nabla_W \phi_t(W_t) \in \rr^{m \times n}$ is the backpropagated (negative) gradient matrix. Then the regular pre-training weight update can be written down as follows ($\eta$ is the learning rate):
\begin{equation}
    W_T = W_0 + \eta \sum_{t=0}^{T-1} \tilde G_{t} = W_0 + \eta\sum_{t=0}^{T-1} \rho_t(G_t)
\end{equation}
where $\tilde G_t$ is the final processed gradient to be added to the weight matrix and $\rho_t$ is an entry-wise stateful gradient regularizer (e.g., Adam). The state of $\rho_t$ can be memory-intensive. For example, for Adam, we need $M,V \in \rr^{m\times n}$ to regularize the gradient $G_t$ into $\tilde G_{t}$:
\begin{eqnarray}
    M_t &=& \beta_1 M_{t-1} + (1-\beta_1) G_t \\  
    V_t &=& \beta_2 V_{t-1} + (1-\beta_2) G^2_t  \\
    \tilde G_t &=& M_t / \sqrt{V_t + \epsilon}
\end{eqnarray}
Here $G_t^2$ and $M_t / \sqrt{V_t + \epsilon}$ means element-wise multiplication and division. $\eta$ is the learning rate. Together with $W\in \rr^{m\times n}$, this takes $3mn$ memory. 

\textbf{Low-rank updates.} For a linear layer $W \in \mathbb{R}^{m \times n}$, LoRA and its variants utilize the low-rank structure of the update matrix by introducing a low-rank adaptor $AB$:
\begin{equation}
    W_T = W_0 + B_{T}A_{T},
\end{equation}
where $B \in \mathbb{R}^{m \times r}$ and $A \in \mathbb{R}^{r \times n}$, and $r \ll \min(m, n)$. $A$ and $B$ are the learnable low-rank adaptors and $W_0$ is a fixed weight matrix (e.g., pre-trained weight).

\subsection{Low-Rank Property of Weight Gradient}
While low-rank updates are proposed to reduce memory usage, it remains an open question whether the weight matrix should be parameterized as low-rank. In many situations, this may not be true. For example, in linear regression $\vy = W\vx$, if the optimal $W^*$ is high-rank, then imposing a low-rank assumption on $W$ never leads to the optimal solution, regardless of what optimizers are used. 

\def\beig{\lambda}
\def\ceig{\nu}

\def\bmin{\underline{\beig}}
\def\cmin{\underline{\ceig}}

\def\cN{\mathcal{N}}

Surprisingly, while the weight matrices are not necessarily low-rank, the gradient indeed becomes low-rank during the training for certain gradient forms and associated network architectures. 

\textbf{Reversible networks.} Obviously, for a general loss function, its gradient can be arbitrary and is not necessarily low rank. Here we study the gradient structure for a general family of nonlinear networks known as ``reversible networks''~\cite{tian2020understanding}, which includes not only simple linear networks but also deep ReLU/polynomial networks:

\begin{definition}[Reversiblity~\cite{tian2020understanding}]
A network $\cN$ that maps input $\vx$ to output $\vy = \cN(\vx)$ is \emph{reversible}, if there exists $L(\vx; W)$ so that $\vy= L(\vx; W)\vx$, and the backpropagated gradient $\vg_\vx$ satisfies $\vg_\vx = L^\top(\vx; W) \vg_\vy$, where $\vg_\vy$ is the backpropagated gradient at the output $\vy$. Here $L(\vx;W)$ depends on the input $\vx$ and weight $W$ in the network $\cN$. 
\end{definition}

Please check Appendix~\ref{sec:reversibility} for its properties. For reversible networks, the gradient takes a specific form. 

\begin{restatable}[Gradient Form of reversible models]{theorem}{gradientreversible}
\label{thm:gradientreversible}
Consider a chained reversible neural network $\cN(\vx) := \cN_L(\cN_{L-1}(\ldots\cN_1(\vx)))$ and define $J_l := \mathrm{Jacobian}(\cN_L) \ldots \mathrm{Jacobian}(\cN_{l+1})$ and $\vf_l := \cN_l(\ldots \cN_1(\vx))$. Then the weight matrix $W_l$ at layer $l$ has gradient $G_l$ in the following form for batch size 1:  

\textbf{(a)} For $\ell_2$-objective $\phi := \frac12\|\vy - \vf_L\|_2^2$: 
\begin{equation}
    G_l = \left(J_l^\top \vy - J^\top_l J_l W_l \vf_{l-1}\right)\vf_{l-1}^\top \label{eq:reversible-grad}
\end{equation}

\textbf{(b)} Left $P^\perp_\vone := I - \frac{1}{K}\vone\vone^\top$ be the zero-mean PSD projection matrix. For $K$-way logsoftmax loss $\phi(\vy; \vf_L) := -\log \left( \frac{\exp(\vy^\top \vf_L)}{\vone^\top \exp(\vf_L)}\right)$ with small logits $\|P^\perp_\vone\vf_L\|_\infty \ll \sqrt{K}$: 
\begin{equation}
    G_l = \left(J_lP^\perp_\vone \vy - \gamma K^{-1}J_l^\top P^\perp_\vone J_l W_l \vf_{l-1}\right)\vf_{l-1}^\top
\end{equation}
where $\gamma \approx 1$ and $\vy$ is a data label with $\vy^\top \vone = 1$.
\end{restatable}

\def\dd{\mathrm{d}}
\def\gzeroproj{G_{t_0}^\parallel}
\def\cV{\mathcal{V}}

From the theoretical analysis above, we can see that for batch size $N$, the gradient $G$ has certain structures: $G = \frac{1}{N}\sum_{i=1}^N (A_i - B_i W C_i)$ for input-dependent matrix $A_i$, Positive Semi-definite (PSD) matrices $B_i$ and $C_i$. In the following, we prove that such a gradient will become low-rank during training in certain conditions:

\def\sr{\mathrm{sr}}

\begin{restatable}[Gradient becomes low-rank during training]{lemma}{gradientlowrank}
\label{lemma:gradientlowrank}
    Suppose the gradient follows the parametric form: 
    \begin{eqnarray}
          G_t=\frac{1}{N}\sum_{i=1}^N (A_i-B_i W_t C_i)\label{eq:constantgradientcoeff}
    \end{eqnarray} 
    with constant $A_i$, PSD matrices $B_i$ and $C_i$ after $t \ge t_0$. We study vanilla SGD weight update: $W_t=W_{t-1}+\eta G_{t-1}$. Let $S := \frac{1}{N}\sum_{i=1}^N C_i \otimes B_i$ and $\lambda_1 < \lambda_2$ its two smallest distinct eigenvalues. Then the stable rank $\sr(G_t)$ satisfies:
    \begin{equation}
        \sr(G_t) \le \sr(\gzeroproj)\!+\!\left(\frac{1\!-\!\eta \lambda_2}{1\!-\!\eta \lambda_1}\right)^{2(t-t_0)} \frac{\|G_0\!-\!\gzeroproj\|_F^2}{\|\gzeroproj\|_2^2} \label{eq:stable-rank-decay}
    \end{equation}
    where $\gzeroproj$ is the projection of $G_{t_0}$ onto the minimal eigenspace $\cV_1$ of $S$ corresponding to $\lambda_1$.
\end{restatable}
In practice, the constant assumption can approximately hold for some time, in which the second term in Eq.~\ref{eq:stable-rank-decay} goes to zero exponentially and the stable rank of $G_t$ goes down, yielding low-rank gradient $G_t$. The final stable rank is determined by $\sr(\gzeroproj)$, which is estimated to be low-rank by the following: 
\begin{restatable}[Low-rank $G_t$]{corollary}{lowrankmid}
\label{co:low-rank-mid}
If the gradient takes the parametric form $G_t = \frac{1}{N}\sum_{i=1}^N (\va_i - B_i W_t \vf_i)\vf_i^\top$ with all $B_i$ full-rank, and $N' := \rank(\{\vf_i\}) < n$, then $\sr(\gzeroproj) \le n - N'$ and thus $\sr(G_t) \le n/2$ for large $t$. 
\end{restatable}
\textbf{Remarks.} The gradient form is justified by  Theorem~\ref{thm:gradientreversible}. Intuitively, when $N'$ is small, $G_t$ is a summation of $N'$ rank-1 update and is naturally low rank; on the other hand, when $N'$ becomes larger and closer to $n$, then the training dynamics has smaller null space $\cV_1$, which also makes $G_t$ low-rank. The full-rank assumption of $\{B_i\}$ is reasonable, e.g., in LLMs, the output dimensions of the networks (i.e., the vocabulary size) is often huge compared to matrix dimensions. 

In general if the batch size $N$ is large, then it becomes a bit tricky to characterize the minimal eigenspace $\cV_1$ of $S$. On the other hand, if $\cV_1$ has nice structure, then $\sr(G_t)$ can be bounded even further: 
\begin{restatable}[Low-rank $G_t$ with special structure of $\cV_1$]{corollary}{lowrankhigh}
If $\cV_1(S)$ is 1-dimensional with decomposable eigenvector $\vv = \vy \otimes \vz$, then $\sr(\gzeroproj) = 1$ and thus $G_t$ becomes rank-1.
\end{restatable}

One rare failure case of Lemma~\ref{lemma:gradientlowrank} is when $\gzeroproj$ is precisely zero, in which $\sr(\gzeroproj)$ becomes undefined. This happens to be true if $t_0 = 0$, i.e., $A_i$, $B_i$ and $C_i$ are constant throughout the entire training process. Fortunately, for practical training, this does not happen. 

\def\vdelta{\boldsymbol{\Delta}}

\textbf{Transformers.} For Transformers, we can also separately prove that the weight gradient of the lower layer (i.e., \emph{project-up}) weight of feed forward network (FFN) becomes low rank over time, using the JoMA framework~\cite{tian2023joma}. Please check Appendix (Sec.~\ref{sec:transformer-low-rank}) for details.

\subsection{Gradient Low-rank Projection (\lowrank{})}
Since the gradient $G$ may have a low-rank structure, if we can keep the gradient statistics of a small ``core'' of gradient $G$ in optimizer states, rather than $G$ itself, then the memory consumption can be reduced substantially. This leads to our proposed \lowrank{} strategy: 
\begin{definition}[Gradient Low-rank Projection (\textbf{\lowrank})]
Gradient low-rank projection (\textbf{\lowrank}) denotes the following gradient update rules ($\eta$ is the learning rate):
\begin{equation}
    \label{eq:represent_low_rank_updates}
    W_T = W_0 + \eta\sum_{t=0}^{T-1} \tilde G_{t}, \quad \tilde G_t = P_t \rho_t(P_t^\top G_t Q_t) Q^\top_t
\end{equation}
where $P_t \in \mathbb{R}^{m \times r}$ and $Q_t \in \mathbb{R}^{n\times r}$ are projection matrices. 
\end{definition}
Different from LoRA, \lowrank{} \textit{explicitly utilizes the low-rank updates} instead of introducing additional low-rank adaptors and hence does not alter the training dynamics. 

In the following, we show that \lowrank{} converges under a similar (but more general) form of gradient update rule (Eqn.~\ref{eq:constantgradientcoeff}). This form corresponds to Eqn.~\ref{eq:reversible-grad} but with a larger batch size.  

\def\ee{\mathbb{E}}
\def\cW{\mathcal{W}}
\def\cN{\mathcal{N}}

\begin{definition}[$L$-continuity]
A function $\vh(W)$ has (Lipschitz) $L$-continuity, if for any $W_1$ and $W_2$, $\|\vh(W_1) - \vh(W_2)\|_F \le L\|W_1-W_2\|_F$. 
\end{definition}

\begin{restatable}[Convergence of \lowrank with fixed projections]{theorem}{convgpg}
    \label{thm:convgpg}
    Suppose the gradient has the form of Eqn.~\ref{eq:constantgradientcoeff} and $A_i$, $B_i$ and $C_i$ have $L_A$, $L_B$ and $L_C$ continuity with respect to $W$ and $\|W_t\|\le D$. Let $R_t := P_t^\top G_t Q_t$, $\hat B_{it} := P_t^\top B_{i}(W_t) P_t$, $\hat C_{it} := Q_t^\top C_i(W_t) Q_t$ and $\kappa_t := \frac1N \sum_i \lambda_{\min}(\hat B_{it}) \lambda_{\min}(\hat C_{it})$. If we choose constant $P_t = P$ and $Q_t=Q$, then \lowrank{} with $\rho_t \equiv 1$ satisfies:
    \begin{equation}
        \|R_t\|_F \le \left[1\!-\!\eta(\kappa_{t-1}\!-\!L_A\!-\!L_B L_C D^2)\right]\|R_{t-1}\|_F \label{eq:converge-rt}
    \end{equation}
    As a result, if $\min_t \kappa_t > L_A + L_B L_C D^2$, $R_t \rightarrow 0$ and thus \lowrank{} converges with fixed $P_t$ and $Q_t$.
\end{restatable}
\textbf{Setting $P$ and $Q$}. The theorem tells that $P$ and $Q$ should project into the subspaces corresponding to the first few largest eigenvectors of $\hat B_{it}$ and $\hat C_{it}$ for faster convergence (large $\kappa_t$). While all eigenvalues of the positive semidefinite (PSD) matrix $B$ and $C$ are non-negative, some of them can be very small and hinder convergence (i.e., it takes a long time for $G_t$ to become $0$). With the projection $P$ and $Q$, $P^\top B_{it} P$ and $Q^\top C_{it} Q$ only contain the largest eigen subspaces of $B$ and $C$, improving the convergence of $R_t$ and at the same time, reduces the memory usage. 

While it is tricky to obtain the eigenstructure of $\hat B_{it}$ and $\hat C_{it}$ (they are parts of Jacobian), one way is to instead use the spectrum of $G_t$ via Singular Value Decomposition (SVD): 
\begin{align}
    \label{eq:svd_p_q}
    G_t &= U S V^{\top} \approx \sum_{i=1}^{r} s_{i} u_{i} v_{i}^{\top} \\
    P_t &= [u_1, u_2, ..., u_r], 
    \quad
    Q_t = [v_1, v_2, ..., v_r]
\end{align}

\textbf{Difference between \lowrank{} and LoRA.}
While both \lowrank{} and LoRA have ``low-rank'' in their names, they follow very different training trajectories. For example, when $r = \min(m, n)$, \lowrank{} with $\rho_t \equiv 1$ follows the exact training trajectory of the original model, as $\tilde G_t = P_t P_t^{\top} G_t Q_t Q_t^\top = G_t$. On the other hand, when $BA$ reaches full rank (i.e., $B \in \mathbb{R}^{m \times m}$ and $A \in \mathbb{R}^{m \times n}$), optimizing $B$ and $A$ simultaneously follows a very different training trajectory compared to the original model.

\section{\lowrank{} for Memory-Efficient Training}
For a complex optimization problem such as LLM pre-training, it may be difficult to capture the entire gradient trajectory with a single low-rank subspace. One reason is that the principal subspaces of $B_t$ and $C_t$ (and thus $G_t$) may change over time. In fact, if we keep the same projection $P$ and $Q$, then the learned weights will only grow along these subspaces, which is not longer full-parameter training. Fortunately, for this, \lowrank{} can switch subspaces during training and learn full-rank weights without increasing the memory footprint.

\subsection{Composition of Low-Rank Subspaces}
\label{sec:composition-subspace}
\begin{figure}
    \centering
    \includegraphics[width=0.58\linewidth]{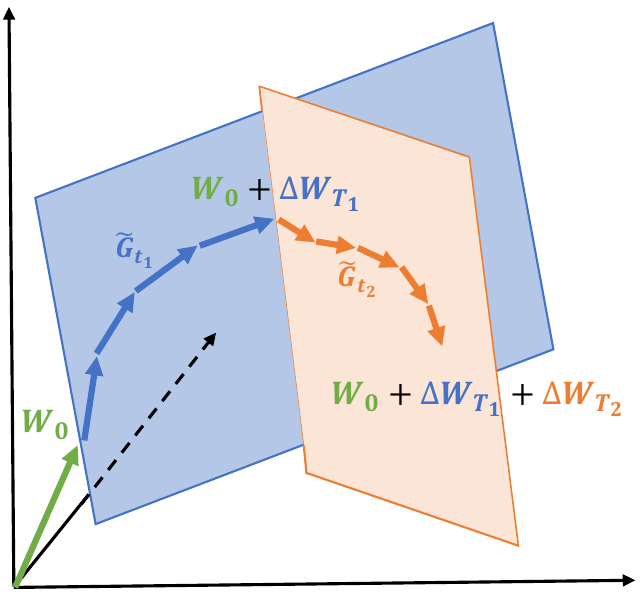}
    \caption{\small{ Learning through low-rank subspaces $\Delta W_{T_1}$ and $\Delta W_{T_2}$ using \lowrank{}. For $t_1 \in [0, T_1 - 1]$, $W$ are updated by projected gradients $\tilde G_{t_1}$ in a subspace determined by fixed $P_{t_1}$ and $Q_{t_1}$. After $T_1$ steps, the subspace is changed by recomputing $P_{t_2}$ and $Q_{t_2}$ for $t_2 \in [T_1, T_2 - 1]$, and the process repeats until convergence.}}
    \label{fig:subspace_learning}
\vspace{-3mm}
\end{figure}

We allow \lowrank{} to switch across low-rank subspaces:
\begin{equation}
    \label{eq:represent_low_rank_updates_multiple}
    W_t = W_0 + \Delta W_{T_1} + \Delta W_{T_2} + \ldots + \Delta W_{T_n},
\end{equation}
where $t \in \left[\sum_{i=1}^{n-1} T_i, \sum_{i=1}^{n} T_i\right]$ and $\Delta W_{T_i} = \eta\sum_{t=0}^{T_i-1} \tilde{G_t}$ is the summation of all $T_i$ updates within the $i$-th subspace.
When switching to $i$-th subspace at step $t=T_i$, we re-initialize the projector $P_t$ and $Q_t$ by performing SVD on the current gradient $G_t$ by Equation \ref{eq:svd_p_q}.
We illustrate how the trajectory of $\tilde{G_t}$ traverses through multiple low-rank subspaces in Fig.~\ref{fig:subspace_learning}.
In the experiment section, we show that allowing multiple low-rank subspaces is the key to achieving the successful pre-training of LLMs.

Following the above procedure, the switching frequency $T$ becomes a hyperparameter. The ablation study (Fig.~\ref{fig:ablation}) shows a sweet spot exists. A very frequent subspace change increases the overhead (since new $P_t$ and $Q_t$ need to be computed) and breaks the condition of constant projection in Theorem~\ref{thm:convgpg}. In practice, it may also impact the fidelity of the optimizer states, which accumulate over multiple training steps. On the other hand, a less frequent change may make the algorithm stuck into a region that is no longer important to optimize (convergence proof in Theorem~\ref{thm:convgpg} only means good progress in the designated subspace, but does not mean good overall performance). While optimal $T$ depends on the total training iterations and task complexity, we find that a value between $T=50$ to $T=1000$ makes no much difference. Thus, the total computational overhead induced by SVD is negligible ($< 10\%$) compared to other memory-efficient training techniques such as memory offloading \citep{rajbhandariZeROMemoryOptimizations2020}.

\subsection{Memory-Efficient Optimization}

\newlength\myindent
\setlength\myindent{2em}
\newcommand\bindent{%
  \begingroup
  \setlength{\itemindent}{\myindent}
  \addtolength{\algorithmicindent}{\myindent}
}
\newcommand\eindent{\endgroup}

\begin{algorithm}[tb]
   \caption{Adam with \lowrank}
   \label{alg:low_rank_adam}
 \begin{algorithmic}
   \STATE {\bfseries Input:} A layer weight matrix $W \in \mathbb{R}^{m \times n}$ with $m \leq n$. Step size $\eta$, scale factor $\alpha$, decay rates $\beta_1, \beta_2$, rank $r$, subspace change frequency $T$.
   \STATE Initialize first-order moment $M_0 \in \mathbb{R}^{n \times r} \gets 0$
   \STATE Initialize second-order moment $V_0 \in \mathbb{R}^{n \times r} \gets 0$
   \STATE Initialize step $t \gets 0$
   \REPEAT
   \STATE $G_t \in \mathbb{R}^{m \times n} \gets - \nabla_W \phi_t(W_t)$ 
   \IF{$t \bmod T = 0$}
   \STATE $U, S, V \gets \text{SVD}(G_t)$
   \STATE $P_t \gets U[:, :r]$ \hfill \COMMENT{Initialize left projector as $m \leq n$}
   \ELSE
   \STATE $P_t \gets P_{t-1}$ \hfill \COMMENT{Reuse the previous projector}
   \ENDIF
   \STATE $R_t \gets P_{t}^{\top} G_t$ \hfill \COMMENT{Project gradient into compact space}
   \\\hrulefill
   \STATE {\bfseries $\update(R_t)$ by Adam}
   \bindent
   \hspace{\algorithmicindent} \STATE $M_t \gets \beta_1 \cdot M_{t-1} + (1 - \beta_1) \cdot R_t$ 
   \hspace{\algorithmicindent} \STATE $V_t \gets \beta_2 \cdot V_{t-1} + (1 - \beta_2) \cdot R_t^2$ 
   \hspace{\algorithmicindent} \STATE $M_t \gets M_t / (1 - \beta_1^t)$
   \hspace{\algorithmicindent} \STATE $V_t \gets V_t / (1 - \beta_2^t)$ 
   \hspace{\algorithmicindent} \STATE $N_t \gets M_t / (\sqrt{V_t} + \epsilon)$
   \eindent
   \\\hrulefill
   \STATE $\tilde G_t \gets \alpha \cdot P N_t$ \hfill \COMMENT{Project back to original space}
   \STATE $W_t \gets W_{t-1} + \eta \cdot \tilde G_t$
   \STATE $t \gets t + 1$
   \UNTIL{convergence criteria met}
   \RETURN $W_t$
 \end{algorithmic}
\end{algorithm}

\textbf{Reducing memory footprint of gradient statistics.} \lowrank{} significantly reduces the memory cost of optimizer that heavily rely on component-wise gradient statistics, such as Adam \citep{kingmaAdamMethodStochastic2014}. 
When $\rho_t \equiv \mathrm{Adam}$, by projecting $G_t$ into its low-rank form $R_t$, Adam's gradient regularizer $\rho_t(R_t)$ only needs to track low-rank gradient statistics.
where $M_t$ and $V_t$ are the first-order and second-order momentum, respectively. 
\lowrank{} computes the low-rank normalized gradient $N_t$ as follows:
\begin{equation}
    \label{eq:low_rank_normalized_gradient}
    N_t = \rho_t(R_t) = M_t / (\sqrt{V_t} + \epsilon).
\end{equation}
\lowrank{} can also apply to other optimizers (e.g., Adafactor) that have similar update rules and require a large amount of memory to store gradient statistics. 
\paragraph{Reducing memory usage of projection matrices.} To achieve the best memory-performance trade-off, we only use one project matrix $P$ or $Q$, projecting the gradient $G$ into $P^\top G$ if $m \leq n$ and $G Q$ otherwise. We present the algorithm applying \lowrank{} to Adam in Algorithm~\ref{alg:low_rank_adam}. 

With this setting, \lowrank{} requires less memory than LoRA during training.
As \lowrank{} can always merge $\Delta W_t$ to $W_0$ during weight updates, it does not need to store a separate low-rank factorization $BA$.
In total, \lowrank{} requires $(mn + mr + 2nr)$ memory, while LoRA requires $(mn + 3mr + 3nr)$ memory.
A comparison between \lowrank{} and LoRA is shown in Table~\ref{tab:lora_compare}.

As Theorem \ref{thm:convgpg} does not require the projection matrix to be carefully calibrated, we can further reduce the memory cost of projection matrices by quantization and efficient parameterization, which we leave for future work.

\subsection{Combining with Existing Techniques}

\lowrank{} is compatible with existing memory-efficient optimization techniques.
In our work, we mainly consider applying \lowrank{} with 8-bit optimizers and per-layer weight updates.

\paragraph{8-bit optimizers.}
\citet{dettmers8bitOptimizersBlockwise2021} proposed 8-bit Adam optimizer that maintains 32-bit optimizer performance at a fraction of the memory footprint.
We apply \lowrank{} directly to the existing implementation of 8-bit Adam.

\paragraph{Per-layer weight updates.}
In practice, the optimizer typically performs a single weight update for all layers after backpropagation. This is done by storing the entire weight gradients in memory. 
To further reduce the memory footprint during training, we adopt per-layer weight updates to \lowrank, which performs the weight updates during backpropagation. This is the same technique proposed in recent works to reduce memory requirement \citep{lvAdaLomoLowmemoryOptimization2023,lvFullParameterFinetuning2023}.

\subsection{Hyperparameters of \lowrank{}}
\label{sec:lowrank-hyperparams}

In addition to Adam's original hyperparameters, \lowrank{} only introduces very few additional hyperparameters: the rank $r$ which is also present in LoRA, the subspace change frequency $T$ (see Sec.~\ref{sec:composition-subspace}), and the scale factor $\alpha$. 

Scale factor $\alpha$ controls the strength of the low-rank update, which is similar to the scale factor $\alpha/r$ appended to the low-rank adaptor in \citet{huLoRALowRankAdaptation2021}.
We note that the $\alpha$ does not depend on the rank $r$ in our case. 
This is because, when $r$ is small during pre-training, $\alpha/r$ significantly affects the convergence rate, unlike fine-tuning.

\begin{table}[t]
    \caption{\small{Comparison between \lowrank{} and LoRA. Assume $W \in \mathbb{R}^{m \times n}$ ($m \leq n$), rank $r$.}}
    
    \label{tab:lora_compare}
    \begin{center}
    \begin{small}
    \begin{tabular}{lcc}
    \toprule
               & \lowrank{} & LoRA \\
    \midrule
    Weights          & $mn$   & $mn+mr+nr$ \\
    Optim States           & $mr + 2nr$   & $2mr + 2nr$  \\
    \midrule
    Multi-Subspace   & \cmark   & \xmark \\
    Pre-Training   & \cmark   & \xmark \\
    Fine-Tuning   & \cmark   & \cmark \\
    \bottomrule
    \end{tabular}
    \end{small}
    \end{center}
\vspace{-6mm}
\end{table}

\section{Experiments}

\begin{table*}[t]
    \centering
    \caption{\small{Comparison with low-rank algorithms on pre-training various sizes of LLaMA models on C4 dataset. Validation perplexity is reported, along with a memory estimate of the total of parameters and optimizer states based on BF16 format. The actual memory footprint of \lowrank{} is reported in Fig.~\ref{fig:memory_vs_model_size}.}}
    \label{tab:lora_compare_llama}
    \begin{tabular}{lcccc}
    \toprule
               & \textbf{60M} & \textbf{130M} & \textbf{350M} & \textbf{1B} \\
    \midrule
    Full-Rank & 34.06 (0.36G) & 25.08 (0.76G) & 18.80 (2.06G) & 15.56 (7.80G) \\
    \midrule
    \textbf{\lowrank} & \textbf{34.88} (0.24G) & \textbf{25.36} (0.52G) & \textbf{18.95} (1.22G) & \textbf{15.64} (4.38G) \\
    Low-Rank & 78.18 (0.26G) & 45.51 (0.54G) & 37.41 (1.08G) & 142.53 (3.57G) \\
    LoRA & 34.99 (0.36G) & 33.92 (0.80G) & 25.58 (1.76G) & 19.21 (6.17G) \\
    ReLoRA & 37.04 (0.36G) & 29.37 (0.80G) & 29.08 (1.76G) & 18.33 (6.17G) \\
    \bottomrule
    $r / d_{model}$ & 128 / 256 & 256 / 768 & 256 / 1024 & 512 / 2048 \\
    Training Tokens & 1.1B & 2.2B & 6.4B & 13.1B \\ %
    \bottomrule
    \end{tabular}
\end{table*}

We evaluate \lowrank{} on both pre-training and fine-tuning of LLMs. All experiments run on NVIDIA A100 GPUs.\vspace{-2mm}

\begin{table}[t]
    \caption{\small{Pre-training LLaMA 7B on C4 dataset for 150K steps. Validation perplexity and memory estimate are reported.}}
    \label{tab:7b_eval}
    \centering
    \begin{tabular}{l|c|cccc}
    \toprule
    & \textbf{Mem}        & \textbf{40K} & \textbf{80K} & \textbf{120K} & \textbf{150K} \\
    \midrule
    \textbf{8-bit \lowrank{}} & 18G & 17.94 & 15.39 & 14.95 & 14.65 \\
    8-bit Adam & 26G & 18.09 & 15.47 & 14.83 & 14.61 \\
    \midrule
    Tokens (B) & & 5.2 & 10.5 & 15.7 & 19.7 \\
    \bottomrule
    \end{tabular}
    \vskip -5mm
\end{table}

\begin{figure*}[ht]
    \vspace{-3mm}
    \centering
    \includegraphics[width=\linewidth]{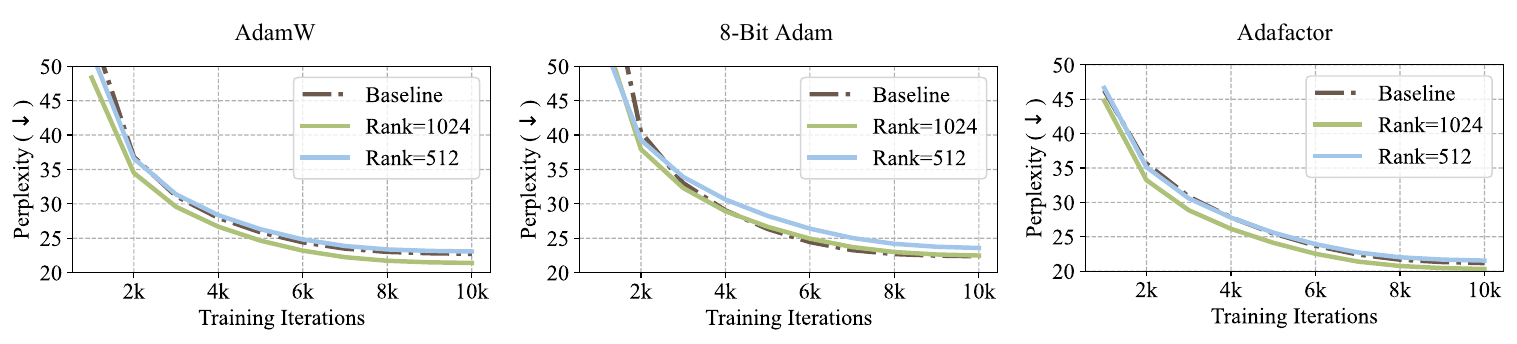}
    \vspace{-7mm}
    \caption{\small{Applying \lowrank{} to different optimizers for pre-training LLaMA 1B on C4 dataset for 10K steps. Validation perplexity over training steps is reported. We apply \lowrank{} to each optimizer with the rank of 512 and 1024, where the 1B model dimension is 2048. }}
    \vspace{-3mm}
    \label{fig:compare_optimizer}
\end{figure*}

\paragraph{Pre-training on C4.}
To evaluate its performance, we apply \lowrank{} to train LLaMA-based large language models on the C4 dataset. 
C4 dataset is a colossal, cleaned version of Common Crawl's web crawl corpus, which is mainly intended to pre-train language models and word representations \citep{raffelExploringLimitsTransfer2023}.
To best simulate the practical pre-training scenario, we train without data repetition over a sufficiently large amount of data, across a range of model sizes up to 7 Billion parameters.
\paragraph{Architecture and hyperparameters.}
We follow the experiment setup from \citet{lialinReLoRAHighRankTraining2023}, which adopts a LLaMA-based\footnote[3]{LLaMA materials in our paper are subject to LLaMA community license.} architecture with RMSNorm and SwiGLU activations \citep{zhangRootMeanSquare2019,shazeerGLUVariantsImprove2020,touvronLlamaOpenFoundation2023}. 
For each model size, we use the same set of hyperparameters across methods, except the learning rate.
We run all experiments with BF16 format to reduce memory usage, and we tune the learning rate for each method under the same amount of computational budget and report the best performance.
The details of our task setups and hyperparameters are provided in the appendix.
\paragraph{Fine-tuning on GLUE tasks.}
GLUE is a benchmark for evaluating the performance of NLP models on a variety of tasks, including sentiment analysis, question answering, and textual entailment \citep{wangGLUEMultiTaskBenchmark2019}.
We use GLUE tasks to benchmark \lowrank{} against LoRA for memory-efficient fine-tuning.

\subsection{Comparison with Existing Low-Rank Methods}

We first compare \lowrank{} with existing low-rank methods using Adam optimizer across a range of model sizes.
\paragraph{Full-Rank}
Our baseline method that applies Adam optimizer with full-rank weights and optimizer states.
\paragraph{Low-Rank}
We also evaluate a traditional low-rank approach that represents the weights by learnable low-rank factorization: $W = BA$ \citep{kamalakaraExploringLowRank2022}.
\paragraph{LoRA}
\citet{huLoRALowRankAdaptation2021} proposed LoRA to fine-tune pre-trained models with low-rank adaptors: $W = W_0 + BA$, where $W_0$ is fixed initial weights and $BA$ is a learnable low-rank adaptor. In the case of pre-training, $W_0$ is the full-rank initialization matrix.
We set LoRA alpha to 32 and LoRA dropout to 0.05 as their default settings.
\paragraph{ReLoRA}
\citet{lialinReLoRAHighRankTraining2023} proposed ReLoRA, a variant of LoRA designed for pre-training, which periodically merges $BA$ into $W$, and initializes new $BA$ with a reset on optimizer states and learning rate. ReLoRA requires careful tuning of merging frequency, learning rate reset, and optimizer states reset. We evaluate ReLoRA without a full-rank training warmup for a fair comparison.

For \lowrank{}, we set subspace frequency $T$ to 200 and scale factor $\alpha$ to 0.25 across all model sizes in Table \ref{tab:lora_compare_llama}.
For each model size, we pick the same rank $r$ for all low-rank methods, and we apply them to all multi-head attention layers and feed-forward layers in the models.
We train all models using Adam optimizer with the default hyperparameters (e.g., $\beta_1=0.9$, $\beta_2=0.999$, $\epsilon=10^{-8}$).
We also estimate the memory usage based on BF16 format, including the memory for weight parameters and optimizer states.
As shown in Table~\ref{tab:lora_compare_llama}, \lowrank{} outperforms other low-rank methods and achieves comparable performance to full-rank training.
We note that for 1B model size, \lowrank{} even outperforms full-rank baseline when $r=1024$ instead of $r=512$.
Compared to LoRA and ReLoRA, \lowrank{} requires less memory for storing model parameters and optimizer states. 
A detailed training setting of each model and memory estimation for each method are in the appendix.

\subsection{\lowrank{} with Memory-Efficient Optimizers}
We demonstrate that \lowrank{} can be applied to various learning algorithms, especially memory-efficient optimizers, to further reduce the memory footprint.
We apply \lowrank{} to AdamW, 8-bit Adam, and Adafactor optimizers \citep{shazeerAdafactorAdaptiveLearning, loshchilovDecoupledWeightDecay2019,dettmers8bitOptimizersBlockwise2021}.
We consider Adafactor with first-order statistics to avoid performance degradation.

We evaluate them on LLaMA 1B architecture with 10K training steps, and we tune the learning rate for each setting and report the best performance.
As shown in Fig.~\ref{fig:compare_optimizer}, applying \lowrank{} does not significantly affect their convergence.
By using \lowrank{} with a rank of 512, the memory footprint is reduced by up to 62.5\%, on top of the memory savings from using 8-bit Adam or Adafactor optimizer.
Since 8-bit Adam requires less memory than others, we denote 8-bit \lowrank{} as \lowrank{} with 8-bit Adam, and use it as the default method for the following experiments on 7B model pre-training and memory measurement.

\subsection{Scaling up to LLaMA 7B Architecture}
Scaling ability to 7B models is a key factor for demonstrating if \lowrank is effective for practical LLM pre-training scenarios.
We evaluate \lowrank{} on an LLaMA 7B architecture with an embedding size of 4096 and total layers of 32.
We train the model for 150K steps with 19.7B tokens, using 8-node training in parallel with a total of 64 A100 GPUs.
Due to computational constraints, we compare 8-bit \lowrank{} ($r=1024$) with 8-bit Adam with a single trial without tuning the hyperparameters.   
As shown in Table \ref{tab:7b_eval}, after 150K steps, 8-bit \lowrank{} achieves a perplexity of 14.65, comparable to 8-bit Adam with a perplexity of 14.61.

\begin{figure}[t!]
    \centering
    \includegraphics[width=0.9\linewidth]{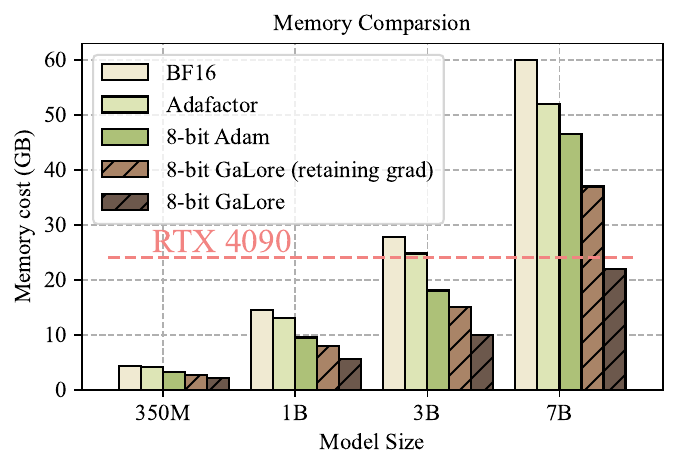}
    \vspace{-4.5mm}
    \caption{\small{Memory usage for different methods at various model sizes, evaluated with a token batch size of 256. 8-bit \lowrank{} (retaining grad) disables per-layer weight updates but stores weight gradients during training.
    }}
    \vspace{-4mm}
    \label{fig:memory_vs_model_size}
\end{figure}

\subsection{Memory-Efficient Fine-Tuning}
\lowrank{} not only achieves memory-efficient pre-training but also can be used for memory-efficient fine-tuning.
We fine-tune pre-trained RoBERTa models on GLUE tasks using \lowrank{} and compare its performance with a full fine-tuning baseline and LoRA.
We use hyperparameters from \citet{huLoRALowRankAdaptation2021} for LoRA and tune the learning rate and scale factor for \lowrank{}.
As shown in Table~\ref{tab:fine_tuning}, \lowrank{} achieves better performance than LoRA on most tasks with less memory footprint.
This demonstrates that \lowrank{} can serve as a full-stack memory-efficient training strategy for both LLM pre-training and fine-tuning.

\begin{table*}[t]
    \caption{\small{Evaluating \lowrank{} for memory-efficient fine-tuning on GLUE benchmark using pre-trained RoBERTa-Base. We report the average score of all tasks.}}
    \label{tab:fine_tuning}
    \centering
    \begin{tabular}{l|c|cccccccc|c} %
    \toprule
               & \textbf{Memory} & \textbf{CoLA} & \textbf{STS-B} & \textbf{MRPC} & \textbf{RTE} & \textbf{SST2} & \textbf{MNLI} & \textbf{QNLI} & \textbf{QQP} & \textbf{Avg} \\
    \midrule
    Full Fine-Tuning & 747M & 62.24 & 90.92 & 91.30 & 79.42 & 94.57 & 87.18 & 92.33 & 92.28 & 86.28 \\
    \midrule
    \textbf{\lowrank{} (rank=4)} & 253M & 60.35 & \textbf{90.73} & \textbf{92.25} & \textbf{79.42} & \textbf{94.04} & \textbf{87.00} & \textbf{92.24} & 91.06 & \textbf{85.89} \\
    LoRA (rank=4) & 257M & \textbf{61.38} & 90.57 & 91.07 & 78.70  & 92.89 & 86.82 & 92.18 & \textbf{91.29} & 85.61 \\
    \midrule
    \textbf{\lowrank{} (rank=8)} & 257M & 60.06 & \textbf{90.82} & \textbf{92.01} & \textbf{79.78} & \textbf{94.38} & \textbf{87.17} & 92.20 & 91.11 & \textbf{85.94} \\
    LoRA (rank=8) & 264M & \textbf{61.83} & 90.80 & 91.90 & 79.06  & 93.46 & 86.94 & \textbf{92.25} & \textbf{91.22} & 85.93 \\
    \bottomrule
    \end{tabular}
    \vskip -0.1in
\end{table*}

\subsection{Measurement of Memory and Throughput}
\label{sec:memory_measure}
While Table~\ref{tab:lora_compare_llama} gives the theoretical benefit of \lowrank{} compared to other methods in terms of memory usage, we also measure the actual memory footprint of training LLaMA models by various methods, with a token batch size of 256. 
The training is conducted on a single device setup without activation checkpointing, memory offloading, and optimizer states partitioning \citep{rajbhandariZeROMemoryOptimizations2020}.

\textbf{Training 7B models on consumer GPUs with 24G memory.} As shown in Fig.~\ref{fig:memory_vs_model_size}, 8-bit \lowrank{} requires significantly less memory than BF16 baseline and 8-bit Adam, and only requires 22.0G memory to pre-train LLaMA 7B with a small per-GPU token batch size (up to 500 tokens). This memory footprint is within 24GB VRAM capacity of a single GPU such as NVIDIA RTX 4090.
In addition, when activation checkpointing is enabled, per-GPU token batch size can be increased up to 4096. While the batch size is small per GPU, it can be scaled up with data parallelism, which requires much lower bandwidth for inter-GPU communication, compared to model parallelism. Therefore, it is possible that \lowrank{} can be used for elastic training~\cite{linDynamicMinibatchSGD2019} 7B models on consumer GPUs such as RTX 4090s. %

Specifically, we present the memory breakdown in Fig.~\ref{fig:memory_breakdown}.
It shows that 8-bit \lowrank{} reduces 37.92G (63.3\%) and 24.5G (52.3\%) total memory compared to BF16 Adam baseline and 8-bit Adam, respectively.
Compared to 8-bit Adam, 8-bit \lowrank{} mainly reduces the memory in two parts: (1) low-rank gradient projection reduces 9.6G (65.5\%) memory of storing optimizer states, and (2) using per-layer weight updates reduces 13.5G memory of storing weight gradients.

\textbf{Throughput overhead of \lowrank{}.} We also measure the throughput of the pre-training LLaMA 1B model with 8-bit \lowrank{} and other methods, where the results can be found in the appendix.
Particularly, the current implementation of 8-bit \lowrank{} achieves 1019.63 tokens/second, which induces 17\% overhead compared to 8-bit Adam implementation. 
Disabling per-layer weight updates for \lowrank{} achieves 1109.38 tokens/second, improving the throughput by 8.8\%. We note that our results do not require offloading strategies or checkpointing, which can significantly impact training throughput. 
We leave optimizing the efficiency of \lowrank{} implementation for future work.

\vspace{-2mm}
\section{Ablation Study}
\paragraph{How many subspaces are needed during pre-training?}

We observe that both too frequent and too slow changes of subspaces hurt the convergence, as shown in Fig.~\ref{fig:ablation} (left). The reason has been discussed in Sec.~\ref{sec:composition-subspace}. In general, for small $r$, the subspace switching should happen more to avoid wasting optimization steps in the wrong subspace, while for large $r$ the gradient updates cover more subspaces, providing more cushion.

\begin{figure}
    \centering
        
    \includegraphics[width=\linewidth]{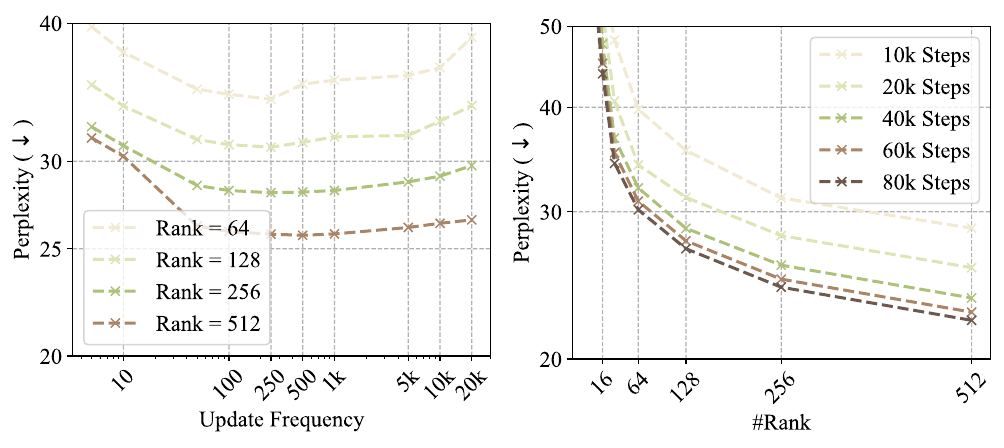}
    \caption{\small{Ablation study of \lowrank{} on 130M models. \textbf{Left:} varying subspace update frequency $T$. \textbf{Right:} varying subspace rank and training iterations.}}
    \vspace{-7mm}
    \label{fig:ablation}
\end{figure}

\paragraph{How does the rank of subspace affect the convergence?}
Within a certain range of rank values, decreasing the rank only slightly affects the convergence rate, causing a slowdown with a nearly linear trend.
As shown in Fig.~\ref{fig:ablation} (right), training with a rank of 128 using 80K steps achieves a lower loss than training with a rank of 512 using 20K steps.
This shows that \lowrank{} can be used to trade-off between memory and computational cost.
In a memory-constrained scenario, reducing the rank allows us to stay within the memory budget while training for more steps to preserve the performance.

\section{Conclusion}

We propose \lowrank{}, a memory-efficient pre-training and fine-tuning strategy for large language models.
\lowrank{} significantly reduces memory usage by up to 65.5\% in optimizer states while maintaining both efficiency and performance for large-scale LLM pre-training and fine-tuning.

We identify several open problems for \lowrank{}, which include (1) applying \lowrank{} on training of various models such as vision transformers \citep{dosovitskiy2021an} and diffusion models \citep{ho2020denoising}, (2) further enhancing memory efficiency by employing low-memory projection matrices, and (3) exploring the feasibility of elastic data distributed training on low-bandwidth consumer-grade hardware.

We hope that our work will inspire future research on memory-efficient training from the perspective of gradient low-rank projection. 
We believe that \lowrank{} will be a valuable tool for the community, enabling the training of large-scale models on consumer-grade hardware with limited resources.

\section*{Impact Statement}

This paper aims to improve the memory efficiency of training LLMs in order to reduce the environmental impact of LLM pre-training and fine-tuning. By enabling the training of larger models on hardware with lower memory, our approach helps to minimize energy consumption and carbon footprint associated with training LLMs.
\section*{Acknowledgments}

We thank Meta AI for computational
support. 
We appreciate the helpful feedback and discussion from Florian Sch{\"a}fer, Jeremy Bernstein, and Vladislav Lialin.
B. Chen greatly appreciates the support by Moffett AI.
Z. Wang is in part supported by NSF Awards 2145346 (CAREER), 02133861 (DMS), 2113904 (CCSS), and the NSF AI Institute for Foundations of Machine Learning (IFML).
A. Anandkumar is supported by the Bren Foundation and the Schmidt Sciences through AI 2050 senior fellow program.

\bibliography{zotero,custom}
\bibliographystyle{icml_template/icml2024}

\newpage
\clearpage
\onecolumn
\appendix
\section{Additional Related Works}

Adafactor \citep{shazeerAdafactorAdaptiveLearning} achieves sub-linear memory cost by factorizing the second-order statistics by a row-column outer product.
\lowrank{} shares similarities with Adafactor in terms of utilizing low-rank factorization to reduce memory cost, but \lowrank{} focuses on the low-rank structure of the gradients, while Adafactor focuses on the low-rank structure of the second-order statistics.

\lowrank{} can reduce the memory cost for both first-order and second-order statistics, and can be combined with Adafactor to achieve further memory reduction. 
In contrast to the previous memory-efficient optimization methods, \lowrank{} operates independently as the optimizers directly receive the low-rank gradients without knowing their full-rank counterparts.

The fused backward operation proposed by LOMO \citep{lvFullParameterFinetuning2023} mitigates the memory cost of storing weight gradients during training.
Integrated with the standard SGD optimizer, LOMO achieves zero optimizer and gradient memory cost during training.
AdaLOMO \citep{lvAdaLomoLowmemoryOptimization2023} enhances this approach by combining the fused backward operation with adaptive learning rate for each parameter, similarly achieving minimal optimizer memory cost.

While LOMO and AdaLOMO represent significant advancements in memory-efficient optimization for fine-tuning or continual pre-training, they might not be directly applicable to pre-training from scratch at larger scales.
For example, the vanilla Adafactor, adopted by AdaLOMO, has been demonstrated to lead to increased training instabilities at larger scales \citep{raeScalingLanguageModels2022,chowdheryPaLMScalingLanguage2022,wortsmanStableLowprecisionTraining,zhaiScalingVisionTransformers2022}.
We believe integrating \lowrank{} with the fused backward operation may offer a promising avenue for achieving memory-efficient large-scale pre-training from scratch.
\def\vec{\mathrm{vec}}

\section{Proofs}

\def\dd{\mathrm{d}}
\def\tr{\mathrm{tr}}
\def\rank{\mathrm{rank}}

\subsection{Reversibility}
\label{sec:reversibility}
\begin{definition}[Reversiblity~\cite{tian2020understanding}]
A network $\cN$ that maps input $\vx$ to output $\vy = \cN(\vx)$ is \emph{reversible}, if there exists $L(\vx; W)$ so that  $\vy= L(\vx; W)\vx$, and the backpropagated gradient $\vg_\vx$ satisfies $\vg_\vx = L^\top(\vx; W) \vg_\vy$, where $\vg_\vy$ is the backpropagated gradient at the output $\vy$. Here $L(\vx;W)$ depends on the input $\vx$ and weight $W$ in the network $\cN$. 
\end{definition}

Note that many layers are reversible, including linear layer (without bias), reversible activations (e.g., ReLU, leaky ReLU, polynomials, etc). Furthermore, they can be combined to construct more complicated architectures: 
\begin{property}
If $\cN_1$ and $\cN_2$ are reversible networks, then (\emph{\textbf{Parallel}}) $\vy = \alpha_1 \cN_1(\vx) + \alpha_2 \cN_2(\vx)$ is reversible for constants $\alpha_1$ and $\alpha_2$, and (\emph{\textbf{Composition}}) $\vy = \cN_2(\cN_1(\vx))$ is reversible. 
\end{property}
From this property, it is clear that ResNet architecture $\vx + \cN(\vx)$ is reversible, if $\cN$ contains bias-free linear layers and reversible activations, which is often the case in practice. For a detailed analysis, please check Appendix A in~\cite{tian2020understanding}. For architectures like self-attention, one possibility is to leverage JoMA~\cite{tian2023joma} to analyze, and we leave for future work.  

The gradient of chained reversible networks has the following structure:
\gradientreversible*
\begin{proof}
Note that for layered reversible network, we have 
\begin{equation}
\cN(\vx) = \cN_L(\cN_{L-1}(...\cN_1(\vx))) = K_L(\vx)K_{L-1}(\vx)\ldots K_1(\vx)\vx 
\end{equation}
Let $\vf_l := \cN_l(\cN_{l-1}(\ldots\cN_1(\vx)))$ and $J_l := K_L(\vx)\ldots K_{l+1}(\vx)$, and for linear layer $l$, we can write $\cN(\vx) = J_lW_l \vf_{l-1}$. Therefore, for the linear layer $l$ with weight matrix $W_l$, we have:
\begin{eqnarray}
    \dd \phi &=& (\vy - \cN(\vx))^\top \dd \cN(\vx) \\
    &=& (\vy - \cN(\vx))^\top K_L(\vx)\ldots K_{l+1}(\vx) \dd W_l \vf_{l-1} \ \ +\ \ \mathrm{terms\ not\ related\ to\ }\dd W_l \\
    &=& (\vy - J_lW_l\vf_{l-1})^\top J_l\dd W_l \vf_{l-1} \\
    &=& \tr(\dd W_l^\top J_l^\top (\vy-J_lW_l\vf_{l-1})\vf^\top_{l-1})
\end{eqnarray}
This gives the gradient of $W_l$:
\begin{equation}
    G_l = J_l^\top \vy \vf^\top_{l-1} - J_l^\top J_l W_l \vf_{l-1}\vf^\top_{l-1} 
\end{equation}
\end{proof}

\textbf{Softmax Case.} Note that for softmax objective with small logits, we can also prove a similar structure of backpropagated gradient, and thus Theorem~\ref{thm:gradientreversible} can also apply.
\begin{restatable}[Gradient structure of softmax loss]{lemma}{gradientsoftmax}
For $K$-way logsoftmax loss $\phi(\vy; \vf) := -\log \left( \frac{\exp(\vy^\top \vf)}{\vone^\top \exp(\vf)}\right)$, let $\hat\vf = P^\perp_\vone \vf$ be the zero-mean version of network output $\vf$, where $P^\perp_\vone := I - \frac{1}{K}\vone\vone^\top$, then we have:
\begin{equation}
   -\dd \phi = \vy^\top \dd\hat\vf - \gamma \hat\vf^\top \dd\hat\vf/K + O(\hat\vf^2/K)\dd\hat\vf  
\end{equation}
where $\gamma(\vy,\vf) \approx 1$ and $\vy$ is a data label with $\vy^\top \vone = 1$.
\end{restatable}
\begin{proof}
Let $\hat\vf := P^\perp_\vone \vf$ be the zero-mean version of network output $\vf$. Then we have $\vone^\top\hat\vf = 0$ and $\vf = \hat\vf + c\vone$. Therefore, we have: 
\begin{equation}
    -\phi = \log \left( \frac{\exp(c)\exp(\vy^\top \hat\vf)}{\exp(c)\vone^\top \exp(\hat\vf)}\right) = \vy^\top\hat\vf - \log(\vone^\top \exp(\hat\vf))
\end{equation}
Using the Taylor expansion $\exp(x) = 1 + x + \frac{x^2}{2} + o(x^2)$, we have: 
\begin{equation}
    \vone^\top \exp(\hat\vf) = \vone^\top(\vone + \hat\vf + \frac12\hat\vf^2) + o(\hat\vf^2) = K (1 + \hat\vf^\top\hat\vf/2K + o(\hat\vf^2/K))
\end{equation}
So
\begin{equation}
    -\phi = \vy^\top\hat\vf - \log(1 + \hat\vf^\top\hat\vf/2K + o(\hat\vf^2/K)) - \log K  
\end{equation}
Therefore
\begin{equation}
    -\dd \phi = \vy^\top\dd \hat\vf - \frac{\gamma}{K} \hat\vf^\top\dd\hat\vf + O\left(\frac{\hat\vf^2}{K}\right)\dd\hat\vf
\end{equation}
where $\gamma := (1 + \hat\vf^\top\hat\vf/2K + o(\hat\vf^2/K))^{-1} \approx 1$. 
\end{proof}
\textbf{Remarks}. With this lemma, it is clear that for a reversible network $\vf := \cN(\vx) = J_l(\vx) W_l\vf_{l-1}(\vx)$, the gradient $G_l$ of $W_l$ has the following form:
\begin{equation}
    G_l = \underbrace{J_lP^\perp_\vone \vy \vf_{l-1}}_A - \underbrace{\gamma J_l^\top P^\perp_\vone J_l}_B W_l \underbrace{\vf_{l-1}\vf_{l-1}^\top / K}_C
\end{equation}

\def\cV{\mathcal{V}}

\subsection{Gradient becomes low-rank}
\gradientlowrank*
\def\vec{\mathrm{vec}}
\begin{proof}
We have
\begin{equation}
    G_t = \frac{1}{N}\sum_{i=1}^N (A_i - B_i W_t C_i) = \frac{1}{N}\sum_{i=1}^N A_i - B_i(W_{t-1} + \eta G_{t-1})C_i = G_{t-1} - \frac{\eta}{N}\sum_{i=1}^N B_i G_{t-1} C_i
\end{equation}
Let $S := \frac{1}{N}\sum_{i=1}^N C_i\otimes B_i$, and $g_t := \vec(G_t) \in \rr^{mn}$ be a vectorized version of the gradient $G_t\in \rr^{m\times n}$. Using $\vec(BWC) = (C^\top \otimes B) \vec(W)$, we have:
\begin{equation}
    g_t = (I - \eta S) g_{t-1} 
\end{equation}
Now let's bound the stable rank of $G_t$:
\begin{equation}
    \text{stable-rank}(G_t) := \frac{\|G_t\|_F^2}{\|G_t\|^2_2} 
\end{equation}
Now $\lambda_1 < \lambda_2$ are the smallest two distinct eigenvectors of $S$. The smallest eigenvalue $\lambda_1$ has multiplicity $\kappa_1$. We can decompose $g_0$ into two components, $g_0 = g^{\parallel}_0 + g^\perp_0$, in which $g^{\parallel}_0$ lies in the $\kappa_1$-dimensional eigenspace $\cV_1$ that corresponds to the minimal eigenvalue $\lambda_1$, and $g^\perp_0$ is its residue. Then $\cV_1 \subset \rr^{mn}$ and its orthogonal complements are invariant subspaces under $S$ and thus: 
\begin{eqnarray}
    \|G_t\|_F^2 &=& \|g_t\|_2^2 = \|(I - \eta S)^t g_{0}\|^2_2 = \|(I - \eta S)^t g^{\parallel}_{0}\|^2_2 + \|(I - \eta S)^t g^{\perp}_{0}\|^2_2 \\
    &\le& (1 - \eta \lambda_2)^{2t} \|g^\perp_0\|_2^2 + (1 - \eta \lambda_1)^{2t} \|g^\parallel_0\|_2^2   
\end{eqnarray}
On the other hand, by our assumption, $G^\parallel_0$ is rank $L$ and thus has SVD decomposition:
\begin{equation}
    G^\parallel_0 = \sum_{l=1}^L c_l \vz_l \vy_l^\top
\end{equation}
with orthonormal unit vectors $\{\vz_l\}_{l=1}^L$ and $\{\vy_l\}_{l=1}^L$ and singular values $\{c_l\}_{l=1}^L1$. This means that 
\begin{equation}
    g^\parallel_0 = \vec(G^\parallel_0) = \sum_{l=1}^L c_l (\vy_l \otimes \vz_l) =: \sum_{l=1}^L c_l \vv_l
\end{equation}
with unit vector $\vv_l := \vy_l \otimes \vz_l \in \cV_1$. It is clear that 
\begin{equation}
\vv^\top_l \vv_{l'} = (\vy^\top_l \otimes \vz^\top_l)(\vy_{l'} \otimes \vz_{l'}) = (\vy^\top_l \vy_{l'})(\vz^\top_l \vz_{l'}) = \mathbb{I}(l=l')
\end{equation}

Therefore, by the definition of spectral norm (or matrix 2-norm), we know it corresponds to the largest singular value, which means:
\begin{eqnarray}
    \|G_t\|_2 &=& \max_{\|\vy'\|_2=1,\|\vz'\|_2=1} \vz^{'\top} G_t \vy' \\
    &\ge& \max_l \vz_l^\top G_t \vy_l = \max_l (\vy_l\otimes \vz_l)^\top g_t \\
    &=& \max_l \vv_l^\top (1 - \eta S)^t g_0 = (1 - \eta \lambda_1)^t \max_l \vv_l^\top g_0
\end{eqnarray}
Note that the last equation is because any $\vv\in \cV_1$ is an eigenvector of $S$ with eigenvalue of $\lambda_1$. 

Since $\vv_l^\top g_0 = \vv_l^\top (g^\perp_0 + g^\parallel_0) = c_l$, $\max_l c_l = \|G^\parallel_0\|_2$ and $\|g^\parallel_0\|_2^2 = \|G^\parallel_0\|_F^2$, we have:
\begin{equation}
    \text{stable-rank}(G_t) := \frac{\|G_t\|_F^2}{\|G_t\|^2_2} \le  \text{stable-rank}(G^\parallel_0) + \left(\frac{1-\eta \lambda_2}{1-\eta \lambda_1}\right)^{2t} \frac{\|G^\perp_0\|_F^2}{\|G_0^\parallel\|_2^2} \label{eq:final-sr-bound}
\end{equation}
\end{proof}

\lowrankmid*
\begin{proof}
Let $C_i = \vf_i\vf_i^\top \in \rr^{n\times n}$. Since $N' := \rank(\{\vf_i\}_{i=1}^N) < n$ and $f_i \in \rr^n$, the collections of vectors $\{\vf_i\}_{i=1}^N$ cannot span the entire space $\rr^n$. Let $\{\vu_j\}_{j=1}^{n-N'}$ be the orthonormal bases for the null space of $\{\vf_i\}_{i=1}^N$, and $\{\ve_k\}_{k=1}^m$ be any orthonormal bases for $\rr^m$. Then the product bases $\{\vu_j\otimes \ve_k\}$ form a set of bases for the minimal eigenspace $\cV_1$ of $S$ with the minimal eigenvalue of $0$. Since $B_i$ are full-rank, no extra dimensions exist for $\cV_1$.

Therefore, when we project $G_{t_0}$ onto $\cV_1$, we have:
\begin{equation}
    \gzeroproj = \sum_{j=1}^{n-N'}\sum_{k=1}^m c_{jk} \vu_j \ve^\top_k = \sum_{j=1}^{n-N'} 
 \vu_j \left(\sum_{k=1}^m c_{jk} \ve_k\right)^\top 
\end{equation}
and thus $\sr(\gzeroproj) \le \rank(\gzeroproj) \le n - N'$,  
since stable rank is a lower-bound of the rank. 

On the other hand, $G_t$ can be written as a summation of $N'$ rank-1 matrices, by representing each $\vf_i = \sum_{j=1}^{N'} b_{ij} \vf'_j$ as a linear combination of $\{\vf'_j\}_{j=1}^{N'}$:  
\begin{equation}
    G_t = \frac1N \sum_{i=1}^N (\va_i - B_i W_t \vf_i)\left(\sum_{j=1}^{N'} b_{ij} \vf'_j\right)^\top = \frac1N \sum_{j=1}^{N'} \left[\sum_{i=1}^N b_{ij} (\va_i - B_i W_t \vf_i)\right] \vf^{'\top}_j 
\end{equation}
and thus has rank at most $N'$. Therefore, when $t$ is sufficiently large so that the second term in Eqn.~\ref{eq:final-sr-bound} is negligible, by Lemma~\ref{lemma:gradientlowrank}, we have (notice that $N' < n$):
\begin{equation}
    \sr(G_t) \le \min(n - N', N') \le n / 2
\end{equation}
\end{proof}

\lowrankhigh*
\begin{proof}
In this case, we have $g^\parallel_0 = \vv\vv^\top g_0 \propto \vv$. Since $\vv = \vy \otimes \vz$, the resulting $G^\parallel_0$ is a rank-1 matrix and thus $\sr(\gzeroproj) = 1$.
\end{proof}

\subsection{Gradient Low-rank property for Transformers}
\label{sec:transformer-low-rank}
Note that Transformers do not belong to the family of reversible networks. However, we can still show that the gradient of the lower layer (i.e., \emph{project-up}) weight $W \in \rr^{m\times n}$ of feed forward network (FFN) becomes low rank over time, using the JoMA framework~\cite{tian2023joma}. Here $m$ is the embedding dimension, and $n$ is the number of hidden nodes in FFNs.
\begin{restatable}[Gradient of Project-up in Transformer FFNs]{lemma}{gradientlowranktransformer}
Suppose the embedding matrix $U \in \rr^{m \times M}$ is fixed and column-orthonormal ($M$ is vocabulary size), the activation functions are linear and the backpropagated gradient are stationary~\cite{tian2023joma}, then the training dynamics of transformed project-up matrix $V := U^\top W \in \rr^{M\times n}$ satisfies the following:
\begin{equation}
    \dot V = \frac{1}{A} \diag\left(\exp\left(\frac{V \circ V}{2}\right)\vone \right)\Delta \label{eq:V-dynamics}
\end{equation}
where $A$ is the normalization factor of softmax, $\circ$ is the Hadamard (element-wise) product and $\Delta$ is defined in the proof. As a result, the gradient of $V$ is ``exponentially more low-rank'' than $V$ itself.  
\end{restatable}
\begin{proof}
Let $\Delta := [\vdelta_1, \ldots, \vdelta_n] \in \rr^{M \times n}$, where $\vdelta_j := \mathbb{E}_{q}[g_j \vx] \in \rr^{M}$. Here $g_j$ is the backpropagated gradient of hidden node $j$ in FFN layer, $\mathbb{E}_q[\cdot]$ is the conditional expectation given the query is token $q$, and $\vx$ is the representation of token distribution in the previous layer of Transformer. Specifically, for intermediate layer, $\vx$ represents the activation output of the previous project-up layer; for the first layer, $\vx$ represents the frequency count of the input tokens. Then following the derivation of Theorem 2~\cite{tian2023joma}, we have for each hidden node $j$ and its weight $\vw_j$, the transformed weight $\vv_j := U^\top \vw_j$ satisfies the following dynamics: 
\begin{equation}
    \dot \vv_j = \frac{1}{A} \vdelta_j \circ \exp(\vv_j ^2 / 2)  
\end{equation}
where $\vv^2_j := \vv_j \circ \vv_j$ is the element-wise square of a vector and $\circ$ is the Hadamard (element-wise) product. Since $V := [\vv_1, \ldots, \vv_n]$, Eqn.~\ref{eq:V-dynamics} follows. 

Note that the dynamics of $\vv_j$ shows that the direction of $\vv_j$ will change over time (because of $\exp(\vv_j^2/2)$), and it is not clear how such dynamics leads to low-rank $V$ and even more low-rank $\dot V$. For this, we per-row decompose the matrix $V$:
\begin{equation}
    V := \left[\begin{array}{c}
       \vu_1^\top \\
       \vu_2^\top \\
       \ldots \\
       \vu_M^\top 
    \end{array}\right]
\end{equation}
where $\vu_l \in \rr^n$. We can also do the same for $\Delta$:
\begin{equation}
    \Delta := \left[
    \begin{array}{c}
       \vmu_1^\top \\
       \vmu_2^\top \\
       \ldots \\
       \vmu_M^\top 
    \end{array}\
    \right] 
\end{equation}
where $\vmu_l \in \rr^n$. Then Eqn.~\ref{eq:V-dynamics} can be decomposed along each row:
\begin{equation}
    \dot \vu_l = \frac{1}{A} (e^{\vu^2_l} \cdot \vone)\vmu_l
\end{equation}
Then it is clear that $\vu_l$ is always along the direction of $\vmu_l$, which is a fixed quality since the backpropagated gradient $g_j$ and input $\vx$ are assumed to be stationary (and thus  $\vdelta_j := \mathbb{E}_q[g_j\vx]$ is a constant). 

Therefore, let $\vu_l(t) = \alpha_l(t) \vmu_l$ with initial condition of the magnitude $\alpha_l(0) = 0$, and we have:
\begin{equation}
    \dot \alpha_l = \frac{1}{A}  e^{\alpha_l^2 \vmu_l^2}\cdot \vone = \frac{1}{A} \sum_{j=1}^n e^{\alpha_l^2 \mu^2_{lj}} \label{eq:alpha-dyn}
\end{equation}
where $1\le l\le M$ is the token index. In the following we will show that for different $l$, the growth of $\alpha_l$ can be very different. This leads to very different row norms of $V$ and $\dot V$ over time, leading to their low-rank structures. Note that Eqn.~\ref{eq:alpha-dyn} does not have a close form solution, instead we could estimate its growth:
\begin{equation}
    \frac{1}{A} e^{\alpha_l^2 \bar\mu^2_l}
    \le \dot \alpha_l \le \frac{n}{A} e^{\alpha_l^2 \bar\mu^2_l}
\end{equation}
where $\bar\mu^2_l := \max_j \mu^2_{lj}$. 

\def\erf{\mathrm{erf}}

Note that both sides have analytic solutions using Gaussian error functions $\erf(x) = \frac{2}{\sqrt{\pi}}\int_0^x e^{-t^2}\dd t \in [-1, 1]$. Specifically, for dynamic system like $\dot x = C e^{\beta^2 x^2}$, we have
\begin{equation}
    e^{-\beta^2 x^2} \dd x = C \dd t 
\end{equation}
which gives:
\begin{equation}
    \frac{\sqrt{\pi}}{2\beta} \erf\left(\beta x(t)\right) = 
    \int_0^{x(t)} e^{-\beta^2 y^2} \dd y = C t 
\end{equation}
or 
\begin{equation}
    x(t) = \frac{1}{\beta} \erf^{-1}\left( \frac{2\beta C}{\sqrt{\pi}}t\right)
\end{equation}

For inequality like $\dot x \ge C e^{\beta^2 x^2}$ or $\dot x \le C e^{\beta^2 x^2}$, similar equation can be derived. Plug that in, we have:
\begin{equation}
    \frac{1}{\bar\mu_l} \erf^{-1}\left(\frac{2\bar\mu_l}{A\sqrt{\pi}}t \right)
    \le \alpha_l(t) \le \frac{1}{\bar\mu_l} \erf^{-1}\left(\frac{2n\bar\mu_l}{A\sqrt{\pi}}t \right)
\end{equation}
Let 
\begin{equation}
h(t;a) := \frac{1}{a}\erf^{-1}\left(\frac{2}{\sqrt{\pi}}\frac{a}{A}t\right)   
\end{equation}
then $\lim_{t\rightarrow A \sqrt{\pi} / 2a } h(t;a) = +\infty$, and $h(t;\bar\mu_l) \le \alpha_l(t) \le n h(t; n \bar\mu_l)$. 

Let $l^* = \arg\max_l \bar\mu_l^*$ be the row with the largest entry of $\mu$, then if $\bar\mu_l^* > n\bar\mu_l$ for all $l\neq l^*$, then when $t \rightarrow t^* := \frac{A\sqrt{\pi}}{2\bar\mu_l^*}$, the magnitude $\alpha_{l^*}(t) \ge h(t;\bar\mu_{l^*}) \rightarrow +\infty$, while $\alpha_l(t) \le n h (t; n\bar\mu_l)$ still stay finite, since its critical time $t' := \frac{A\sqrt{\pi}}{2n\bar\mu_l} > t^*$. Since $\alpha_l(t)$ controls the magnitude of each row of $V$, This means that $V$ eventually becomes rank-1 and so does $W$. 

Finally, $\dot V$ is even more low rank than $V$, since $\dot \alpha_l$ has $\alpha_l$ in its exponents. 
\end{proof}

\subsection{Convergence of \lowrank{}}
\convgpg*
\begin{proof}
Using $\vec(AXB) = (B^\top \otimes A)\vec(X)$ where $\otimes$ is the Kronecker product, the gradient assumption can be written as the following:
\begin{equation}
    g_t = a_t - S_t w_t 
\end{equation}
where $g_t := \vec(G_t) \in \rr^{mn}$, $w_t := \vec(W_t) \in\rr^{mn}$ be the vectorized versions of $G_t$ and $W_t$, $a_t := \frac1N\sum_i \vec(A_{it})$ and $S_t = \frac1N\sum_i C_{it} \otimes B_{it}$ are $mn$-by-$mn$ PSD matrix. 

Using the same notation, it is clear to show that:
\begin{eqnarray}
    (Q\otimes P)^\top g_t &=& (Q^\top \otimes P^\top) \vec(G_t) = \vec(P^\top G_t Q) = \vec(R_t) =: r_t \\
    \tilde g_t := \vec(\tilde G_t) &=& \vec(PP^\top G_t QQ^\top) = (Q\otimes P)\vec(R_t) = (Q\otimes P)r_{t} 
\end{eqnarray}

Then we derive the recursive update rule for $g_t$:
\begin{eqnarray}
    g_t &=& a_t - S_t w_t \\
    &=& (a_t - a_{t-1}) + (S_{t-1} - S_t) w_t + a_{t-1} - S_{t-1}w_t \\ 
    &=& e_t + a_{t-1} - S_{t-1}(w_{t-1} + \eta \tilde g_{t-1}) \\
    &=& e_t + g_{t-1} - \eta S_{t-1} \tilde g_{t-1}  
\end{eqnarray}
where $e_t := (a_t - a_{t-1}) + (S_{t-1} - S_t) w_t$. Left multiplying by $(Q\otimes P)^\top$, we have: 
\begin{eqnarray}
    r_t = (Q\otimes P)^\top e_t + r_{t-1} - \eta (Q\otimes P)^\top S_{t-1} (Q\otimes P)r_{t-1} 
\end{eqnarray}
Let 
\begin{equation}
 \hat S_t := (Q\otimes P)^\top S_t (Q\otimes P) = \frac1N \sum_i (Q\otimes P)^\top (C_{it} \otimes B_{it}) (Q\otimes P) = \frac1N \sum_i (Q^\top C_{it}Q) \otimes (P^\top B_{it} P)   
\end{equation}
Then we have:
\begin{equation}
    r_t = (I - \eta \hat S_{t-1})r_{t-1} + (Q\otimes P)^\top e_t
\end{equation}
Now we bound the norm. Note that since $P$ and $Q$ are projection matrices with $P^\top P = I$ and $Q^\top Q = I$, we have: 
\begin{equation}
\|(Q\otimes P)^\top e_t\|_2 = \|\vec(P^\top E_t Q)\|_2 = \|P^\top E_t Q\|_F \le \|E_t\|_F
\end{equation}
where $E_t := \frac1N\sum_i (A_{it} - A_{i,t-1}) + \frac1N\sum_i (B_{i,t-1} W_t C_{i,t-1} - B_{it} W_t C_{it})$. So we only need to bound $\|E_t\|_F$. Note that:
\begin{eqnarray}
    \|A_t - A_{t-1}\|_F &\le& L_A \|W_t - W_{t-1}\|_F = \eta L_A \|\tilde G_{t-1}\|_F \le \eta L_A \|R_{t-1}\|_F \\
    \|(B_t - B_{t-1})W_t C_{t-1}\|_F &\le& L_B \|W_t - W_{t-1}\|_F \|W_t\|_F \|C_{t-1}\|_F = \eta L_B L_C D^2 \|R_{t-1}\|_F \\ 
    \|B_t W_t (C_{t-1} - C_t)\|_F &\le& L_C  \|B_t\|_F \|W_t\|_F\|W_{t-1} - W_t\|_F = \eta L_B L_C D^2 \|R_{t-1}\|_F 
\end{eqnarray}

Now we estimate the minimal eigenvalue of $\hat S_{t-1}$. Let $\bmin_{it} := \lambda_{\min}(P^\top B_{it} P)$ and $\cmin_{it} := \lambda_{\min}(Q^\top C_{it} Q)$, then $\lambda_{\min}((P^\top B_{it} P) \otimes (Q^\top C_{it} Q)) = \bmin_{it}\cmin_{it}$ and for any unit vector $\vv$: 
\begin{equation}
    \vv^\top \hat S_t \vv = \frac1N \sum_i \vv^\top \left[(P^\top B_{it} P) \otimes (Q^\top C_{it} Q)\right]\vv \ge \frac1N \sum_i \bmin_{it}\cmin_{it} 
\end{equation}
And thus $\lambda_{\min}(\hat S_t) \ge \frac1N \sum_i \bmin_{it}\cmin_{it}$. Therefore, $\lambda_{\max}(I - \eta \hat S_{t-1}) \le 1 - \frac{\eta}{N} \sum_i \bmin_{i,t-1}\cmin_{i,t-1}$. Therefore, let $\kappa_t := \frac1N \sum_i \bmin_{it}\cmin_{it}$ and using the fact that $\|r_t\|_2 = \|R_t\|_F$, we have: 
\begin{equation}
    \|R_t\|_F\le \left[1 - \eta (\kappa_{t-1} - L_A - 2L_BL_C D^2)\right] \|R_{t-1}\|_F
\end{equation}
and the conclusion follows. 
\end{proof}

\newpage
\section{Details of Pre-Training Experiment}

\subsection{Architecture and Hyperparameters}
We introduce details of the LLaMA architecture and hyperparameters used for pre-training. 
Table~\ref{tab:llama_hyperparameters} shows the most hyperparameters of LLaMA models across model sizes. 
We use a max sequence length of 256 for all models, with a batch size of 131K tokens.
For all experiments, we adopt learning rate warmup for the first 10\% of the training steps, and use cosine annealing for the learning rate schedule, decaying to 10\% of the initial learning rate.

\begin{table*}[h]
    \caption{Hyperparameters of LLaMA models for evaluation. Data amount are specified in tokens.}
    \label{tab:llama_hyperparameters}
    \vskip 0.15in
    \begin{center}
    \begin{tabular}{cccccccc}
    \toprule
    Params & Hidden & Intermediate & Heads & Layers & Steps & Data amount \\
    \midrule
    60M & 512 & 1376 & 8 & 8  & 10K & $1.3 \mathrm{~B}$ \\
    130M & 768 & 2048 & 12 & 12  & 20K & $2.6 \mathrm{~B}$ \\
    350M & 1024 & 2736 & 16 & 24  & 60K & $7.8 \mathrm{~B}$ \\
    $1 \mathrm{~B}$ & 2048 & 5461 & 24 & 32 & 100K & $13.1 \mathrm{~B}$ \\
    $7 \mathrm{~B}$ & 4096 & 11008 & 32 & 32 & 150K & $19.7 \mathrm{~B}$ \\
    \bottomrule
    \end{tabular}
    \end{center}
    \vskip -0.1in
\end{table*}

For all methods on each size of models (from 60M to 1B), we tune their favorite learning rate from a set of $\{0.01, 0.005, 0.001, 0.0005, 0.0001\}$, and the best learning rate is chosen based on the validation perplexity.
We find \lowrank{} is insensitive to hyperparameters and tends to be stable with the same learning rate across different model sizes.
For all models, \lowrank{} use the same hyperparameters, including the learning rate of $0.01$, scale factor $\alpha$ of $0.25$, and the subspace change frequency of $T$ of $200$. We note that since $\alpha$ can be viewed as a fractional learning rate, most of the modules (e.g., multi-head attention and feed-forward layers) in LLaMA models have the actual learning rate of $0.0025$.
This is, still, a relatively large stable learning rate compared to the full-rank baseline, which usually uses a learning rate $\leq 0.001$ to avoid spikes in the training loss.

\subsection{Memory Estimates}
As the GPU memory usage for a specific component is hard to measure directly, we estimate the memory usage of the weight parameters and optimizer states for each method on different model sizes.
The estimation is based on the number of original parameters and the number of low-rank parameters, trained by BF16 format.
For example, for a 60M model, LoRA ($r=128$) requires $42.7$M parameters on low-rank adaptors and $60M$ parameters on the original weights, resulting in a memory cost of $0.20$G for weight parameters and $0.17$G for optimizer states.
Table~\ref{tab:memory_estimate} shows the memory estimates for weight parameters and optimizer states for different methods on different model sizes, as a compliment to the total memory reported in the main text.

\begin{table*}[h]
    \caption{Memory estimates for weight parameters and optimizer states.}
    \label{tab:memory_estimate}
    \begin{subtable}{.5\linewidth}
        \centering
        \caption{Memory estimate of weight parameters.}
        \begin{tabular}{lcccc}
        \toprule
                   & \textbf{60M} & \textbf{130M} & \textbf{350M} & \textbf{1B} \\
        \midrule
        Full-Rank & 0.12G & 0.25G & 0.68G & 2.60G \\
        \midrule
        \textbf{\lowrank{}} & 0.12G & 0.25G & 0.68G & 2.60G \\
        Low-Rank & 0.08G & 0.18G & 0.36G & 1.19G \\
        LoRA & 0.20G & 0.44G & 1.04G & 3.79G \\
        ReLoRA & 0.20G & 0.44G & 1.04G & 3.79G \\
        \bottomrule
        \end{tabular}
    \end{subtable}%
    \begin{subtable}{.5\linewidth}
        \centering
        \caption{Memory estimate of optimizer states.}
        \begin{tabular}{lcccc}
        \toprule
                   & \textbf{60M} & \textbf{130M} & \textbf{350M} & \textbf{1B} \\
        \midrule
        Full-Rank & 0.23G & 0.51G & 1.37G & 5.20G \\
        \midrule
        \textbf{\lowrank{}} & 0.13G & 0.28G & 0.54G & 1.78G \\
        Low-Rank & 0.17G & 0.37G & 0.72G & 2.38G \\
        LoRA & 0.17G & 0.37G & 0.72G & 2.38G \\
        ReLoRA & 0.17G & 0.37G & 0.72G & 2.38G \\
        \bottomrule
        \end{tabular}
    \end{subtable}
\end{table*}

\newpage
\subsection{Training Progression}
We show the training progression of 130M, 350M, 1B and 7B models in Figure~\ref{fig:loss_curves}.
Compared to LoRA, GaLore closely matches the training trajectory of the full-rank baseline, and it even converges slightly faster at the beginning of the training. 

\begin{figure*}[h]
    \centering
    \includegraphics[width=0.8\textwidth]{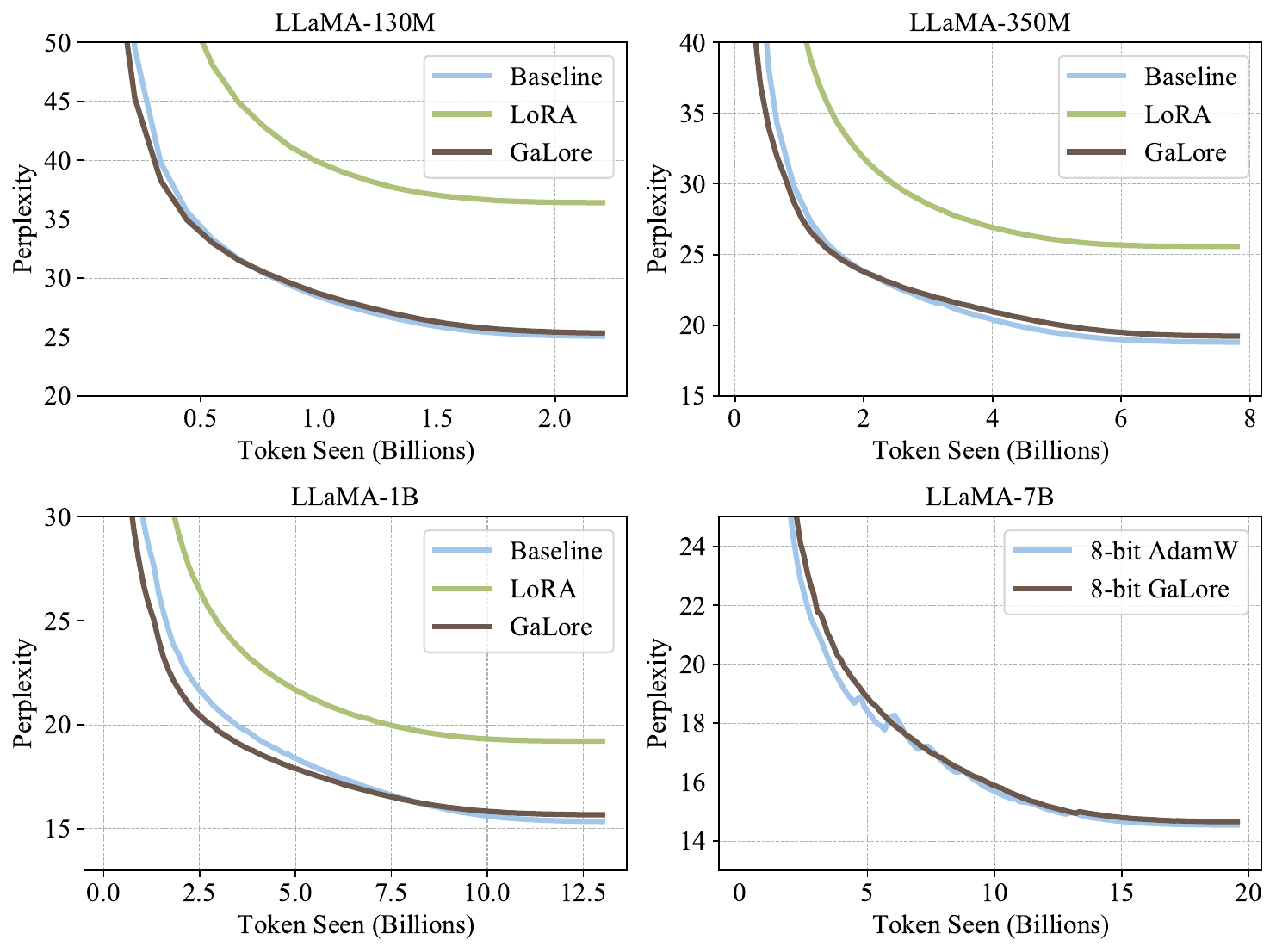}
    \caption{Training progression for pre-training LLaMA models on C4 dataset.}
    \label{fig:loss_curves}
\end{figure*}

\newpage
\section{Fine-Tuning Experiments}

\subsection{Details of Fine-Tuning on GLUE}

We fine-tune the pre-trained RoBERTa-Base model on the GLUE benchmark using the model provided by the Hugging Face\footnote{\url{https://huggingface.co/transformers/model_doc/roberta.html}}.
We trained the model for 30 epochs with a batch size of 16 for all tasks except for CoLA, which uses a batch size of 32.
We tune the learning rate and scale factor for \lowrank{}.
Table~\ref{tab:ft_hyperparameters} shows the hyperparameters used for fine-tuning RoBERTa-Base for \lowrank{}.

\begin{table*}[h]
    \caption{Hyperparameters of fine-tuning RoBERTa base for \lowrank.}
    \centering
    \label{tab:ft_hyperparameters}
    \begin{tabular}{ccccccccc}
    \toprule
    & MNLI   & SST-2 & MRPC    & CoLA    & QNLI    & QQP     & RTE     & STS-B   \\
    \midrule
    Batch Size    & 16     & 16    & 16      & 32      & 16      & 16      & 16      & 16      \\
    \# Epochs     & 30     & 30    & 30      & 30      & 30      & 30      & 30      & 30      \\
    Learning Rate & 1E-05  & 1E-05     & 3E-05   & 3E-05   & 1E-05   & 1E-05   & 1E-05   & 1E-05   \\    
    Rank Config. & & & & $r=4$ & & & & \\
    \lowrank \(\alpha\) & & & & 4 & & & & \\
    Max Seq. Len. & & & & 512 & & & & \\
    \bottomrule
    \end{tabular}
    \vskip 0.1in
    \begin{tabular}{ccccccccc}
    \toprule
    & MNLI   & SST-2 & MRPC    & CoLA    & QNLI    & QQP     & RTE     & STS-B   \\
    \midrule
    Batch Size    & 16     & 16    & 16      & 32      & 16      & 16      & 16      & 16      \\
    \# Epochs     & 30     & 30    & 30      & 30      & 30      & 30      & 30      & 30      \\
    Learning Rate & 1E-05  & 2E-05     & 2E-05   & 1E-05   & 1E-05   & 2E-05   & 2E-05   & 3E-05   \\    
    Rank Config. & & & & $r=8$ & & & & \\
    \lowrank \(\alpha\) & & & & 2 & & & & \\
    Max Seq. Len. & & & & 512 & & & & \\
    \bottomrule
    \end{tabular}
\end{table*}
\subsection{Fine-Tuning on SQuAD dataset}
We evaluate GaLore on the SQuAD dataset \citep{rajpurkarSQuAD1000002016} using the pre-trained BERT-Base model. We use rank $16$ for both GaLore and LoRA. GaLore outperforms LoRA in both Exact Match and F1 scores.
\begin{table*}[h]
    \caption{Evaluating GaLore on SQuAD dataset. Both Exact Match and F1 scores are reported.}
    \label{tab:fine_tuning_squad}
    \centering
    \begin{tabular}{ccc} %
    \toprule
            & \textbf{Exact Match} & \textbf{F1} \\
    \midrule
    Baseline & 80.83 & 88.41 \\
    \midrule
    \textbf{GaLore} & \textbf{80.52} & \textbf{88.29}  \\
    LoRA & 77.99 & 86.11  \\
    \bottomrule
    \end{tabular}
    \vskip -0.1in
\end{table*}

\subsection{Fine-Tuning on OpenAssistant Conversations Dataset}
We apply GaLore on fine-tuning experiments on the OpenAssistant Conversations dataset \citep{kopfOpenAssistantConversationsDemocratizing2023}, using the pre-trained models, including Gemma-2b, Phi-2, and LLaMA-7B \citep{touvronLlamaOpenFoundation2023,gemmateamGemmaOpenModels2024}. We use rank of 128 for both GaLore and LoRA. The results are shown in Table~\ref{tab:fine_tuning_oasst}.

\begin{table*}[h]
    \caption{Evaluating GaLore on OpenAssistant Conversations dataset. Testing perplexity is reported.}
    \label{tab:fine_tuning_oasst}
    \centering
    \begin{tabular}{cccc} %
    \toprule
            & \textbf{Gemma-2b} & \textbf{Phi-2} & \textbf{LLaMA-7B} \\
    \midrule
    Baseline & 4.53 & 3.81 & 2.98 \\
    \midrule
    \textbf{GaLore} & \textbf{4.51} & \textbf{3.83} & 2.95 \\
    LoRA & 4.56 & 4.24 & \textbf{2.94} \\
    \bottomrule
    \end{tabular}
    \vskip -0.1in
\end{table*}

\subsection{Fine-Tuning on Belle-1M Dataset}
We also apply GaLore on fine-tuning experiments on the Belle-1M dataset \citep{BELLE}, using the pre-trained models, including Gemma-2b, Phi-2, and LLaMA-7B. We use rank of 128 for both GaLore and LoRA. The results are shown in Table~\ref{tab:fine_tuning_belle}.
\begin{table*}[h]
    \caption{Evaluating GaLore on Belle-1M dataset. Testing perplexity is reported.}
    \label{tab:fine_tuning_belle}
    \centering
    \begin{tabular}{cccc} %
    \toprule
            & \textbf{Gemma-2b} & \textbf{Phi-2} & \textbf{LLaMA-7B} \\
    \midrule
    Baseline & 5.44 & 2.66 & 2.27 \\
    \midrule
    \textbf{GaLore} & \textbf{5.35} & \textbf{2.62} & \textbf{2.28} \\
    LoRA & 5.37 & 2.75 & 2.30 \\
    \bottomrule
    \end{tabular}
    \vskip -0.1in
\end{table*}
\newpage
\section{Additional Memory Measurements}

We empirically measure the memory usage of different methods for pre-training LLaMA 1B model on C4 dataset with a token batch size of 256, as shown in Table~\ref{tab:memory_measure_1b}.

\begin{table}[!htb]
\centering
    \caption{Measuring memory and throughput on LLaMA 1B model.}
    \label{tab:memory_measure_1b}
    {\begin{tabular}{c|c|c|c|c|cc} \toprule
    \multirow{2}{*}{Model Size} &  \multirow{2}{*}{Layer Wise} & \multirow{2}{*}{Methods} & \multirow{2}{*}{Token Batch Size} & \multirow{2}{*}{Memory Cost} & \multicolumn{2}{c}{Throughput} \\ 
    & & & & & \#Tokens / s & \#Samples / s \\ \hline
    \multirow{4}{*}{1B} &  \multirow{4}{*}{\ding{56}} & AdamW & 256 & 13.60 & 1256.98 & 6.33 \\
     &   & Adafactor & 256 & 13.15 & 581.02 & 2.92 \\
     &   & Adam8bit & 256 & 9.54 & 1569.89 & 7.90 \\
     &   & 8-bit \lowrank{} & 256 & 7.95 & 1109.38 & 5.59 \\ \midrule 
     \multirow{5}{*}{1B} &  \multirow{5}{*}{\ding{52}} & AdamW & 256 & 9.63 & 1354.37 & 6.81 \\
     &   & Adafactor & 256 & 10.32 & 613.90 & 3.09 \\
     &   & Adam8bit & 256 & 6.93 & 1205.31 & 6.07 \\
     &   & 8-bit \lowrank{} & 256 & 5.63 & 1019.63 & 5.13 \\ 
     \bottomrule
    \end{tabular}}
\end{table}

\end{document}